\documentclass[runningheads]{llncs}
\usepackage{subcaption}
\usepackage[usenames,dvipsnames]{xcolor}
\usepackage{tkz-berge}
\usetikzlibrary{fit,shapes}                                     
\usepackage{setspace}
\usepackage{graphicx}

\usepackage{calrsfs}
\DeclareMathAlphabet{\pazocal}{OMS}{zplm}{m}{n}
\usepackage{graphicx}
\usepackage{graphics} 
\usepackage{epsfig} 
\usepackage{mathptmx}
\usepackage{times} 
\usepackage{tikz}
\usetikzlibrary{positioning}
\usepackage{amsmath, amssymb}

\usepackage{amsthm}
\usepackage{amsfonts} 
\usepackage{setspace}
\usepackage{multirow}
\usepackage{textcomp}
\usepackage[english]{babel}
\usepackage{float} 
\usepackage{algorithmicx}
\usepackage[section]{algorithm}
\usepackage{algpascal}
\usepackage{algc}
\usepackage{algcompatible}
\usepackage{algpseudocode}
\let\oldReturn\Return
\renewcommand{\Return}{\State\oldReturn}
\usepackage{mathtools}
\usepackage{bm}
\usepackage{mathptmx}
\usepackage{times} 
\usepackage{mathtools, cuted}
\usepackage{textcomp}
\usepackage[english]{babel}
\usepackage{rotating}
\usepackage{pgfplots}
\pgfplotsset{compat=1.5}
\usepackage{pgfplotstable}
\usepackage{filecontents}
\usepackage{multirow}
\usepackage{booktabs}
\usepackage{array, makecell} %
\usepackage[mathscr]{euscript}

\newtheorem{defn}[theorem]{Definition}
\newtheorem{prop}[theorem]{Proposition}

\graphicspath{{figures/}}
\numberwithin{theorem}{section}

\newcommand{\remove}[1]{}

\def \*{\star}
\def \10n{\!\!\!\!\!\!\!\!\!\!}

\usepackage[utf8]{inputenc}
\begin{document}

\title{Game of Trojans: A Submodular Byzantine Approach}
\author{Dinuka Sahabandu$^1$, Arezoo Rajabi$^1$, Luyao Niu$^1$,\\ Bo Li$^2$, Bhaskar Ramasubramanian$^3$, Radha Poovendran$^1$}
\authorrunning{D. Sahabandu, A. Rajabi, L. Niu, B. Li, B. Ramasubramanian, R. Poovendran}
\institute{$^1$Department of Electrical and Computer Engineering, University of Washington,\\ Seattle, WA 98195. \email{\{sdinuka, rajabia, luyaoniu, rp3\}@uw.edu} \\
$^2$Department of Computer Science, University of Illinois at Urbana-Champaign,\\ Urbana, IL 61801.  \email{lbo@illinois.edu}\\
$^3$Electrical and Computer Engineering, Western Washington University,\\ Bellingham, WA 98225.  \email{ramasub@wwu.edu}}
%
\maketitle
\vspace*{-5 mm}           
\begin{abstract}
Machine learning models in the wild can be downloaded, retrained, and redeployed. 
Such models have been shown to be vulnerable to 
Trojan attacks during training. 
Although different detection mechanisms have been proposed, strong adaptive attackers have been shown to be effective against them. In this paper, we aim to answer the questions considering an intelligent and adaptive adversary: (i) What is the minimal amount of instances that is required to be Trojaned by a strong attacker? and (ii) Is it possible for such an attacker to bypass a strong detection mechanisms?

In particular, we provide an analytical characterization of adversarial capability and the strategic interactions between the adversary and detection mechanism that take place in such models. 
We characterize adversary capability in terms of the fraction of the input dataset that can be embedded with a Trojan trigger. 
We show that the loss function has a submodular structure, which leads to the design of computationally efficient algorithms to determine this fraction with provable bounds on optimality. 
We propose a \emph{Submodular Trojan} algorithm to constructively determine the minimal fraction of samples to inject a Trojan trigger. 
To evade detection of the Trojaned model by potential detectors, we model the strategic interactions between the adversary and Trojan detection mechanism as a two-player game. We show that the adversary wins the game with probability one, thus bypassing detection. 
We establish this by proving that output probability distributions of a Trojan model and a clean model are identical when following our proposed \emph{Min-Max (MM) Trojan} algorithm. 

We perform extensive evaluations of our algorithms 
on MNIST, CIFAR-10, and EuroSAT datasets. 
The results validate our theoretical analysis by demonstrating that (i)
with \emph{Submodular Trojan} algorithm, the adversary needs to embed a Trojan trigger into a very small fraction ($\approx 5\%$) of samples to achieve high accuracy on both Trojan and clean samples (e.g., $97\%$ and $85\%$ for EuroSAT), and 
(ii) the \emph{MM Trojan} algorithm yields a trained Trojan model that evades detection with probability one. 
\end{abstract}

\section{Introduction}
\label{sec:introduction}

Deep neural networks (DNNs) used for machine learning (ML) in data-intensive applications such as computer vision~\cite{krizhevsky2012imagenet}, healthcare~\cite{esteva2019guide}, autonomous driving~\cite{bojarski2016end}, and game-playing~\cite{silver2016mastering} have achieved impressive levels of performance. 
Large amounts of data and extensive computing resources required to train DNNs have resulted in widespread adoption of online ML platforms~\cite{AWS},~\cite{BigML},~\cite{Caffe}. 
These platforms offer ready-to-use ML models and architectures for specific applications (e.g., image classification). 
However, when such `models in the wild' are downloaded, retrained, and uploaded by a user, the training data might contain adversarial examples~\cite{goodfellow2015explain} or Trojans/ trapdoors~\cite{gu2019badnets}, which will cause the model to misbehave. 
Recent research has focused on making large ML models safe for use and explainable~\cite{ullah2019cyber}. 
Preventing insertion of adversarial inputs and detecting embedding of Trojans is a daunting task, since trigger sizes are typically imperceptible, and Trojan samples might form a negligible fraction of the entire training data set. 

Due to the computational complexity of detecting Trojans, novel approaches from the systems design, algorithms, and optimization perspectives, as well as domain knowledge, have been gaining the attention of researchers. 
However, most defenses either assume that (i) model hyperparameters are available~\cite{kolouri2020universal},~\cite{li2021neural},~\cite{liu2018fine},~\cite{yoshida2020disabling}, or (ii) a pre-processing module can be added to the target model~\cite{liu2017neural}. 
Such assumptions restrict a defense strategy to a particular application or class of attacks. 
Solution approaches that make use of reinforcement learning (RL)~\cite{panagiota2020trojdrl} or multi-armed bandit algorithms~\cite{shen2021backdoor} as well as randomized testing~\cite{xu2021detecting} hold high promise. 
These methods combine exploration of the environment and exploitation of current knowledge, and offers generalizability across application domains. 

Among these, the most prominent methodology to detect Trojans is the use of randomized testing strategies. 
Such techniques make use of the observation that given two different DNNs, which for all practical purposes provide the same outputs on known inputs, will not provide the same output when presented with a random (unseen) input selected by the tester. 
An effective randomized testing algorithm called meta-neural Trojan detection (MNTD) was recently proposed in~\cite{xu2021detecting}. 
MNTD used only black-box access to the models, and constructed multiple shadow models (combination of Trojan and non-Trojan models) to learn a binary classifier which determined whether a model was Trojan or not. 
The key insight guiding the success of MNTD was that probability distributions of outputs from Trojan and non-Trojan models will be different, even though both models might have similar accuracies. 
MNTD was shown to have high success rates when evaluated on multiple applications, including speech, natural language processing, and images. 

In~\cite{rajabi2022trojan}, it was observed that if a Trojan embedding were to also be trained either through knowledge distillation~\cite{hinton2015distilling}, or actual training images themselves, then such trained models could effectively bypass random input-based detectors $100\%$ of the time. 
This experimental result begs the question, ``\emph{How to characterize adversary capability and the strategic interactions that take place in such models?}'' 

In this paper, we answer the above question by partitioning it into: 
(i) characterizing adversary capability in terms of the fraction of the input dataset available to use for adversary training, and 
(ii) developing a strategic game-theoretic model characterizing performance, and showing that the results obtained in~\cite{rajabi2022trojan} is a consequence of the game model. 
We consider two types of threat models. 
First, we assume that the adversary can embed a Trojan into a fraction of the dataset, while parameters of the detector remain fixed. 
Second, we consider the case that detector parameters are also being optimized to detect a Trojan model. 
In this case, the adversary is assumed to have knowledge of statistical parameters used by the detector, and the \emph{adversary objective is to bypass detection by the well-trained, optimized detector}. 
We provide bounds on optimality for (i), and detailed experiments evaluating our solutions for (i) and (ii) using the MNIST~\cite{lecun1998mnist}, CIFAR-10~\cite{krizhevsky2009CIFARs}, and EuroSAT~\cite{helber2019eurosat} datasets. 
Our analysis and experiments reveal that the formulation is consistent with results in~\cite{rajabi2022trojan}. 

We make the following key contributions:
\begin{itemize}
    \item When the adversary can access a fraction $\alpha$ of training samples, we obviate the need for RL or bandit algorithms. We accomplish this by showing that the loss function has a submodular structure in $\alpha$. 
    A direct consequence of this formulation is the design of a computationally efficient \emph{Submodular Trojan} algorithm with bounds on optimality. 
    \item We model the strategic interaction between the adversary and a state-of-the-art Trojan detector as a two-player game, and prove that the adversary wins this game with probability one. 
    We establish this by showing that output probability distributions of a Trojan model is identical to that of a clean model when following our proposed \emph{Min-Max Trojan} algorithm, while maintaining high accuracy of classification. 
    \item We carry out extensive experiments on three datasets: MNIST~\cite{lecun1998mnist}, CIFAR-10~\cite{krizhevsky2009CIFARs}, and EuroSAT~\cite{helber2019eurosat}. Our algorithms can also be evaluated on other widely used datasets including CIFAR-100 \cite{krizhevsky2009CIFARs} and GTSRB \cite{stallkamp2012man}.
\end{itemize}

The rest of this paper is organized as follows: Sec. \ref{sec:relatedwork} discusses related work, Sec. \ref{sec:preliminaries} introduces the problem setup, 
Sec. \ref{sec:Submod} formulates and derives the \emph{Submodular Trojan} optimization algorithm, 
Sec. \ref{sec:GameModel} presents the game-theoretic model and the \emph{Min-Max (MM) Trojan} algorithm, Sec. \ref{sec:evaluation} presents experimental results, Sec. \ref{sec:Discussion} discusses extensions 
when some assumptions are relaxed, and Sec. \ref{sec:Conclusion} concludes the paper. 

\section{Related Work}
\label{sec:relatedwork}
We place our contributions in this paper in the context of related work in backdoor attacks on ML models, detecting Trojan models, and mitigating the impact of backdoors. 

\subsection{Trojan Attacks on ML Models}

The authors of~\cite{ji2018transferlearningtrojan} demonstrated that pre-trained ML models might be maintained by untrustworthy entities, consequently resulting in severe security implications to a user of the model. 
They introduced a class of \emph{model-reuse} attacks, wherein a host system running this model could be induced to misbehave on specific inputs in a predictable way. 
A notion of \emph{latent backdoors} was introduced in~\cite{yao2019latent}, which are backdoors that are preserved during a transfer learning process. 
As a result, if a model obtained after the transfer learning procedure includes the target label of the backdoor, it can be activated by an appropriate input. 
This work additionally showed that defending against latent backdoors involves a tradeoff between cost and classification accuracy. 

Model or data poisoning by an adversary has also been shown to be an effective attack strategy. 
An algorithm in~\cite{rakin2020tbt} efficiently identified vulnerable bits of model weights stored in memory. These bits were flipped using existing bit-flipping methods, which transformed a deployed DNN model into a Trojan model. 
In a concurrent work~\cite{dumford2020backdooring}, targeted weight perturbations were used to embed backdoors into convolutional neural networks deployed for face-recognition. 
In contrast to the model poisoning methods described above, the training set of an ML model was poisoned in~\cite{shafahi2018poison}. 
The adversary in this case had the ability to control the outcome of classification for the poisoned data, while performance of the model was maintained for `clean' samples. 

\subsection{Detecting Trojan Models}

An intuitive defense against data poisoning attacks is to develop techniques to remove suspicious samples from the training data. 
Backdoor attacks were shown to leave a distinct \emph{spectral signature} in the covariance of the feature representation learned by a neural network in~\cite{tran2018spectral}. 
The spectral signature, in addition to detecting corrupted training examples, also presents a barrier to the crafting and design of backdoors to ML models. 
Gradients of loss function at the input layer of a neural network were used to extract signatures to identify and eliminate poisoned data samples in~\cite{chan2019poison} even when the target class and ratio of poisoned samples were unknown. 
In situations where the spectral signature of poisoned samples was not large enough to be detected, the authors of~\cite{hayase2021spectre} used robust covariance
estimation to amplify the spectral signature of poisoned data to enhance their detectability. 

When the ML model has an embedded Trojan, it is natural to design methods to detect such models by identifying potential triggers. 
A method to detect and reverse engineer embedded triggers in DNNs called \emph{NeuralCleanse} was proposed in~\cite{wang2019neural}. 
\emph{NeuralCleanse} devised an optimization scheme to determine the smallest trigger required to misclassify all samples from all labels to the target label. 
However, the target label may not be known, and the requirement to consider all labels as a potential target label makes this a computationally expensive procedure. 
An improvement on \emph{NeuralCleanse} called \emph{TABOR} was developed in~\cite{guo2019tabor}. 
\emph{TABOR} used a combination of heuristics and regularization techniques to reduce false positives in Trojan detection. 
Differences in explanations of outputs of a clean model and a Trojan model, even on clean samples, were used to identify Trojan models in~\cite{huang2019neuroninspect}. 
This was achieved without requiring access to samples containing a trigger. 

A Trojan detection method called \emph{DeepInspect} proposed in~\cite{chen2019deepinspect} used conditional generative models to learn probability distributions of potential triggers while having only black-box access to the deployed ML model and without requiring access to clean training data. 
To overcome a limitation of~\cite{chen2019deepinspect} (and of~\cite{wang2019neural}) that there is exactly one target label, the authors of~\cite{xu2021detecting} proposed Meta-Neural Trojan Detection (MNTD). 
MNTD only required black-box access to models and trained a meta-classifier to identify if the model was Trojan or not. 
Multiple shadow models (combination of clean and Trojan models) were generated and their representations were used to learn a binary classifier. 
At test-time, the representation of the target model was provided to the classifier to determine whether it was a Trojan model or not. 
MNTD was demonstrated to have a high success rate when evaluated on a diverse range of datasets, including images, speech, and NLP. 

\subsection{Mitigating Impact of Backdoors}

Once an embedded Trojan has been detected, a question arises if methodologies can be developed to remove or reduce the impact of a backdoor, while maintaining the performance of the ML model on clean inputs. 
The \emph{Fine-pruning} method in~\cite{liu2018fine} proposes to deactivate neurons that are not enabled by clean inputs. It then uses a tuning procedure to restore some of the deactivated neurons to mitigate the reduction in classification accuracy due to pruning. 
An empirical approach to determine different values and types of noise that could suppress a backdoor and classify an image to its true prediction was presented in~\cite{sarkar2020backdoor}. 
The optimal number of noisy copies required for majority voting to improve classification accuracy of backdoored images was then computed. 
A run-time Trojan attack detection system called STRIP was developed in~\cite{gao2019strip}. 
The \emph{input-agnostic} property of a trigger was exploited as a weakness of Trojan attacks. 
The entropy of predicted labels of a model arising from random perturbations to inputs to the model was used to identify the presence of a backdoor. 
A low value of entropy usually corresponded to the presence of a backdoor, since the presence of a trigger in the input was likely to classify the input to a specific target class. 
This method was independent of model architecture and size of trigger. 
A comprehensive survey of backdoor attacks and defenses against such attacks can be found in~\cite{gao2020backdoor}. 

\subsection{Bypassing Trojan Detection}

The authors of~\cite{rajabi2022trojan} observed that Trojan detection mechanisms that use a randomized set of inputs at test time can be effectively bypassed if the Trojan embedding was trained using actual training images, or through knowledge distillation. 
In this paper, we characterize adversary capability and strategic interactions between the Trojan model and detector to provide an analytical framework for empirical results presented in~\cite{rajabi2022trojan}. 

\section{Setup}\label{sec:preliminaries}
This section provides a brief description of our setup. 
We describe data poisoning attacks on ML systems and methods to detect Trojan models. 
The objective in this paper is for the adversary to (i) minimize the fraction of samples into which a Trojan needs to be embedded, and (ii) evade detection by an adaptive detection mechanism. 
The adversary simultaneously aims to maintain high accuracy of classifying clean samples correctly and classifying Trojan samples to the target class. 

\subsection{Data Poisoning Attacks against DNN Classifiers}\label{subsec:ATA}

Let $\{C_1,\cdots,C_k \}$ be a set of classes and $\mathscr{D}=\{(x_i,Y_i^*)\}$ be a dataset, where $x_i\in \mathbb{R}^{d\times d}$ is the $i^{th}$ sample in the dataset, and $Y_i^*$ is a $k$-dimensional vector such that the $j^{th}$ entry $Y_i^{*,j}=1$ if $x_i$ is of class $j\in\{1,\ldots,k\}$ and $Y_i^{*,j}=0$ otherwise. 
The vector $Y_i^*$ correctly labels each sample $x_i$ to its class. Using the dataset, Deep Neural Network (DNN) classifiers are trained to predict the most relevant class among $k$ possible classes for a given input.  
The DNN classifier outputs a function $f(x;\theta)$ for input sample $x$, where $\theta$ represents hyperparameters of the DNN. 
Output $f(x; \theta)$ is a probability vector of dimension $k$, whose $j^{th}$ entry, denoted as $f^j(x; \theta)$ gives the probability that $x$ belongs to class $j$. 
The sample $x$ is assigned to the class that has the highest probability in $f(x;\theta)$, i.e., $x\in C_i$ if $f^i(x; \theta)>f^j(x; \theta) \:\: \forall j\neq i$.

DNN classifiers have been shown to be vulnerable to data poisoning attacks. 
An adversary can embed a Trojan trigger $\delta$ into a sample $x$ to modify it as $x' = x \times (1-\Delta)+ \delta \times \Delta$, where $\Delta$ is a mask specifying the location of the trigger. By poisoning a subset samples in the dataset, the adversary can bias the classifier such that the classifier returns the true class for samples that do not contain the trigger and returns a predefined output label $Y_T$ for Trojan samples that contain the trigger $\delta$.

We denote the set of modified samples as $\mathscr{D}_T=  \{(x_i', Y_T) | x_i \in \mathscr{D}, x_i' = x_i \times (1-\Delta)+ \delta \times \Delta \}$ and the overall dataset as $\mathscr{D}_{p}=\mathscr{D}_T \cup \mathscr{D} $. 
To train a Trojan model, the adversary minimizes a loss function that measures the difference between the expected output (the true label for clean samples and desired label for Trojan samples) and the output of the model on the dataset $\mathscr{D}_{p}$ of size $|\mathscr{D}_p|$:
\begin{equation}\label{EqnCrossEnt}
    \min_{\theta} \frac{1}{|\mathscr{D}_p|}\sum_{(x_i,y_i) \in  \mathscr{D}_p} \mathcal{L}_{CE}(Y_i,f(x_i;\theta)), 
\end{equation}
where $\mathcal{L}_{CE}(Y_i,f(x_i;\theta)):=-Y_i'\log(f(x_i;\theta))$ is the cross-entropy, and logarithm is taken element-wise. 
Note that to train a clean model the adversary substitutes $\mathscr{D}_p$ with $\mathscr{D}$.

\subsection{Trojan Model Detection}\label{subsec:MNTD}

In order to detect Trojan models, we use an insight based on MNTD~\cite{xu2021detecting} that Trojan embedded models and clean models yield similar outputs for clean samples, while having different decision boundaries for randomly generated samples. 
To distinguish a Trojan model from a clean (no Trojan embedded) model, MNTD trains a set of a clean models with the same structure as the target model and a set of Trojan models for different triggers and desired outputs. 
It then learns a discriminator on the output of these models for random image inputs. 
The discriminator was shown to be able to achieve high accuracy in detecting Trojan models. 
While MNTD uses a \emph{model-based} discriminator, the discriminator in our framework (detailed in Section \ref{sec:GameModel}) is \emph{instance-based}. 
The Trojan model detector in this paper is a binary classifier learned using multiple outputs from clean and Trojan models. 
\section{Submodular Algorithm for Optimal Trojan Embedding}
\label{sec:Submod}

Neural networks that comprise deep ML models are described by nonlinear functions, which makes it extremely challenging for an adversary to simultaneously determine the fraction of data into which a Trojan needs to be embedded and optimize hyperparameters of the Trojan model. 
In this section, we show that when the Trojan model hyperparameters are fixed, the worst-case loss function for the adversary has a supermodular structure. 
This property guides the design of a computationally efficient \emph{Submodular Trojan} algorithm to determine the optimal fraction $\alpha$ of the data into which the Trojan needs to be embedded.
We will then show how such attacker can bypass Trojan model detectors in Sec. \ref{sec:GameModel}.

\subsection{Attacker Strategy}\label{Subsec:Threat} 

We assume that the adversary can embed a Trojan into a fraction of the dataset, while detector parameters remain fixed. 
We consider a situation where an adversary can download and train a publicly available dataset, but can perturb or poison only a fraction $\alpha$ of samples. 
The objective of the adversary is to embed a backdoor into the model by poisoning as few samples as possible. 
Our intuition is that the insertion of a backdoor comes at a cost to the adversary- e.g., a resource cost, revealing its presence, etc. 
The adversary is assumed to have access to adequate computational resources to train a local model and estimate the portion of data that needs to be poisoned. We assume that all samples are equivalent in the sense that embedding a Trojan into any subset of $S$ samples will result in the same magnitude of the loss function. 
We provide an analysis of this threat model in Section \ref{subsec:Submod}. 

\subsection{Analysis and Algorithm}\label{subsec:Submod}
We let $f_T(\cdot; \theta_{T})$ denote the Trojan model, parameterized by $\theta_{T}$.
Let $\mathscr{D}_C$ denote the clean data available to the adversary from the dataset of interest, $\mathscr{D}_T \subset \mathscr{D}_C$ be the samples in $\mathscr{D}_C$ into which the Trojan is embedded by the adversary. Let $|\mathscr{D}_C|=N$, and $|\mathscr{D}_T|=\lfloor \alpha N\rfloor$. 
The fraction $\alpha$ is also termed the \emph{poisoing ratio}.
We use $Y_{T}$ and $Y_{C}$ to denote the probability distribution vectors to label Trojaned samples and clean samples, respectively. 
Let $x_T^i$ and $x_C^i$ denote the $i^{th}$ sample in $\mathscr{D}_T$ and the $i^{th}$ sample in $\mathscr{D}_C$, respectively. 
We assume that the samples are generated uniformly and that the adversary uses the cross-entropy loss to train the Trojaned model $f_T(\cdot; \theta_{T})$. The adversary aims to minimize a loss function:
\begin{multline}\label{eq:alpha_prob}
    \min_{\theta_{T}, \alpha} \hspace{1mm} F_{T}(\theta_T, \alpha) =   \underbrace{\frac{1}{\alpha N}\sum_{i=1}^{\lfloor \alpha N \rfloor} \mathcal{L}_{CE}(Y_T,f_T(x_i;\theta_T))}_{F_T^A(\theta_T, \alpha)} \\+ \underbrace{\frac{1}{(1-\alpha)N}\sum_{i=1}^{\lceil (1-\alpha) N \rceil} \mathcal{L}_{CE}(Y_C,f_T(x_i;\theta_T))}_{F_T^C(\theta_T, \alpha)},
\end{multline}
where $F_T^A(\theta_T, \alpha)$ and $F_T^C(\theta_T, \alpha)$ are the loss induced by Trojan embedding and training over clean samples, respectively. In Eqn. (\ref{eq:alpha_prob}), $\lfloor \alpha N \rfloor$ and $\lceil (1-\alpha) N \rceil$ are used to enforce the constraint that the summations are carried out over integers. 
Jointly optimizing over $\theta_T$ and $\alpha$ is a difficult problem to solve since $\theta_T$ implicitly depends on the fraction of data $\alpha$ that has been embedded with the Trojan. 
Instead, we develop a framework for the adversary to alternate between computing $\alpha$ and $\theta_T$. At each iteration, the adversary computes $\alpha$ via the \emph{Submodular Trojan} algorithm with $\theta_T$ fixed, which will be detailed in Algorithm \ref{alg:greedy}, and uses this value of $\alpha$ to optimize $F_T$ over $\theta_T$. 
For arbitrary neural network architectures, it is difficult to explicitly identify the mapping from $\alpha$ to $F_T(\theta_T,\alpha)$. 
To this end, we first derive a bound on 
$F_T(\theta_T,\alpha)$ for any $\theta_T$ and $\alpha$. 
We then show how $\alpha$ impacts the bound.
\begin{prop}\label{lemma:UB}
For any given $\theta_T$ and $\alpha$, there exists a function $\bar{F}_T(\theta_T,\alpha)\geq F_T(\theta_T,\alpha)$, such that $$\bar{F}_T(\theta_T,\alpha)=-\frac{1}{\alpha N}\sum_{i=1}^{N} \mathcal{L}_{CE}(Y_T,f_T(x_i;\theta_T)) \\- \frac{1}{(1-\alpha)N}\sum_{i=1}^{N} \mathcal{L}_{CE}(Y_C,f_T(x_i;\theta_T)).$$
\end{prop}
\begin{proof}
By definition, models $f_T$ and $f_C$ give the probability distributions over classes. Thus, we have that $f_T(x;\theta),f_C(x;\theta)\in[0,1]^{k}$, where $k$ is the number of classes. By the definition of logarithm, we have that 
\begin{align*}
&-\frac{1}{\alpha N}\sum_{i=1}^{\lfloor\alpha N\rfloor} (Y_T' \log (f_T(x^i_T; \theta_T)))\leq -\frac{1}{\alpha N}\sum_{i=1}^{N} (Y_T' \log (f_T(x^i_T; \theta_T))),\\
&- \frac{1}{(1-\alpha)N}\sum_{i=1}^{\lceil (1-\alpha)N\rceil} (Y_C' \log (f_T(x^i_C; \theta_T)))\leq - \frac{1}{(1-\alpha)N}\sum_{i=1}^{ N} (Y_C' \log (f_T(x^i_C; \theta_T))),
\end{align*}
which yields the desired result.
\end{proof}

Proposition \ref{lemma:UB} establishes an upper bound $\bar{F}_T(\theta_T,\alpha)$ for the loss function of the adversary. 
In the following, we show that $\bar{F}_T(\theta_T,\alpha)$ is supermodular in $\alpha$ (or, $-\bar{F}_T(\theta_T,\alpha)$ is submodular in $\alpha$). 
Our intuition for exploring the supermodular property of the loss function is as follows: let $0<\alpha\ll 1$ be such that the adversary chooses only a very small fraction of clean data into which to embed Trojans. 
After the model has been trained with this training set, during inference, the model will have a smaller loss value corresponding to clean data samples and a much higher loss corresponding to Trojaned data. When the model is retrained with a larger value of $\alpha$, the inference loss corresponding to clean data will increase while that corresponding to Trojaned data will decrease. 
Therefore, we posit that gradually increasing $\alpha$ will result in a threshold value at which the total inference loss will be minimum. 
We state a definition of DR-submodularity from~\cite{BiaMirBuh-16} that will be used to prove the main result of this section.  

\begin{defn}[Table~II in \cite{BiaMirBuh-16}]\label{def:DR1}
A function $f(\cdot)$ defined over $\pazocal{X} \in \mathbb{R}^{n}$ satisfies the diminishing returns (DR) property if for all $a\leqslant b \in \pazocal{X}$, for all $i\in n$, for all $k\in \mathbb{R}_+$ such that  $(k\pazocal{X}_i+a)$ and $(k\pazocal{X}_i+b)$ are still in $\pazocal{X}$, it holds, $f(k\pazocal{X}_i +a)- f(a) \geqslant f(k\pazocal{X}_i +b)-f(b)$. $f(\cdot)$ is called a DR-submodular function. 
Moreover, when $f$ is twice differentiable function, DR-submodularity is equivalent to $\nabla^2_{ij}f(x) \leqslant 0$, for all $i, j$.
\end{defn}

If a function $F$ is submodular, then $-F$ is supermodular. The following result proves that $F_T(\theta_T, \alpha)$ is supermodular by showing that $-F_T(\theta_T, \alpha)$ is DR-submodular in $\alpha$. 

\begin{theorem}\label{thm:FDR}
The function $\bar{F}_{T}(\theta_T, \alpha)$ defined in Proposition \ref{lemma:UB} is supermodular in $\alpha$.
\end{theorem}
\begin{proof}
We prove the statement by showing that $-\bar{F}_{T}(\theta_T,\alpha)$ is DR-submodular in $\alpha$. Taking the second derivative of $-\bar{F}_{T}(\theta_T,\alpha)$ yields
\begin{align}
     \frac{2}{\alpha^3 N}\sum_{i=1}^{N}(Y_T' \log (f_T(x^i_T; \theta_T))) + \frac{2}{(1-\alpha)^3 N}\sum_{i=1}^{N} (Y_C' \log (f_T(x^i_C; \theta_T))) \leq 0. \nonumber
\end{align}
By Definition~\ref{def:DR1}, $-\bar{F}_{T}(\theta_T, \alpha)$ is DR-submodular with respect to $\alpha$, and thus $\bar{F}_{T}(\theta_T, \alpha)$ is supermodular in $\alpha$.
\end{proof}

Theorem \ref{thm:FDR} enables design of a computationally efficient \emph{Submodular Trojan} algorithm to approximate the optimal value of $\alpha$~\cite{bian2017continuous}. 
Algorithm \ref{alg:greedy} describes the procedure. 
The algorithm begins with an initial value of $\alpha = \alpha^0$. 
At each stage, $\alpha$ is increased by a stepsize $\gamma^t $, and the difference between the values of $\bar{F}_T$ at successive iterations is compared. The 
algorithm terminates when the cumulative step size $c^t \geq 1$. 
The following result characterizes the worst-case cost $\bar{F}_T(\theta_T,\alpha)$. 

\begin{algorithm}[h]
  \caption{Submodular Trojan
  \label{alg:greedy}}
  \begin{algorithmic}
  
\State \textit {\bf Input:} Clean dataset $\mathscr{D}_C$, Clean model function $f_C(\cdot, \theta_C)$ with its pre-trained weights $\theta_C$, Trojaned model weights $\theta_T$, stepsize $ 0 < \gamma \ll 1$, function $\bar{F}_T$.
\State \textit{\bf Output:} Fraction of clean data to embed Trojan $\alpha$.
\end{algorithmic}
\begin{algorithmic}[1]
\State Initialize $\gamma^0\leftarrow\gamma$, $\alpha^{0} \leftarrow \gamma$, $t \leftarrow 0$, $c^t\leftarrow 0$ \label{step:init}
\While {$c^t <1$}
\State $v^t\leftarrow \underset{v\leq 1-\alpha}{\arg\max}\langle v,-\nabla \bar{F}_T(\theta_T,\alpha)\rangle$
\State $\gamma^t = \gamma$
\State $\gamma^t\leftarrow \min\{\gamma^t,1-c^t\}$
\State $\alpha^{t} \leftarrow \alpha^{t-1} + \gamma^tv^t$
\State $c^{t+1}\leftarrow c^t+\gamma^t$
\EndWhile
\Return $\alpha^{t-1}$ 
\end{algorithmic}
\end{algorithm}

\begin{prop}
Let $\bar{F}_T(\theta_T,\alpha)\in[\lambda,\beta]$. Then the fraction of samples $\alpha$ that contain a Trojan returned by Algorithm \ref{alg:greedy} ensures that $\bar{F}_T(\theta_T,\alpha)\leq e^{-1}\lambda+(1-\frac{1}{e})\beta.$
\end{prop}
\begin{proof}
Using Theorem \ref{thm:FDR}, we have that $\beta-\bar{F}_T(\theta_T,\alpha)$ is non-negative and submodular. 
From \cite[Corollary 1]{bian2017continuous}, we have that Algorithm \ref{alg:greedy} achieves $\beta-\bar{F}_T(\theta_T,\alpha)\geq \frac{1}{e}(\beta-\lambda)$ with sublinear convergence rate. Rearranging the terms yields the desired result.
\end{proof}

The results in this section and the Submodular Trojan algorithm 
assumed that detector parameters remained fixed when optimizing over $\theta_T$ and $\alpha$. 
In practice, detection mechanisms might adapt or tune their parameters in response to adversary behavior. 
The adversary, in turn, will modify its behavior to prevent detection of the Trojan model. 
We subsequently describe how to train $\theta_T$ for the adversary while incorporating interactions with an adaptive detection mechanism. Parameters of the detector will be optimized in response to adversary parameters, motivating development of a game-theoretic framework.
\section{Two-Player Game Model: Min-Max Trojan Algorithm}\label{sec:GameModel}

In this section, we model strategic interactions between the Trojan ML model (adversary) and detection mechanism as a two-player game. 
We show that the detector will not be able to distinguish between outputs of a clean and a Trojan model by proving that solving the game results in the adversary winning with probability one. 
This analytical characterization is consistent with results in~\cite{rajabi2022trojan}, and informs the design of the \emph{Min-Max Trojan} (MM Trojan) algorithm.

\subsection{Threat Model in the Two-Player Game}
In Section \ref{sec:Submod}, we showed how the adversary could compute the optimal fraction of samples into which a Trojan needs to be embedded. 
Now, suppose that the Trojan model is trained by sending data to a possibly untrustworthy third-party (strong adversary) that possesses adequate computational resources. 
A user might use a test dataset (which has been independently evaluated on a clean model) to evaluate the learned model. 
However, the adversary is aware of this possibility and aims to train a model that will be undetectable to random input-based Trojan detection mechanisms. 
In this setting, the adversary is assumed to have knowledge of statistical parameters used by the detector. 
We evaluate three settings: 
\begin{itemize}
\item \underline{\emph{Clean model}} is trained to have a high accuracy of classifying clean samples correctly.
\item \underline{\emph{Baseline Trojan}} is trained to have a high accuracy of (i) classifying clean samples correctly, and (ii) classifying Trojan samples to the target class. The Baseline Trojan is not trained to evade a trained detector, and thus can be identified by an effective state-of-the-art detection mechanism such as MNTD~\cite{xu2021detecting}. 
\item \underline{\emph{MM Trojan}} is trained to accomplish the objective of the Baseline Trojan, and additionally, consistently bypass a trained game-based detector. We characterize and analyze the MM Trojan model in Section \ref{subsec:Game}.
\end{itemize}
While a model with no constraints (Clean) is expected to have the highest accuracy, the accuracy of the Baseline and MM Trojan models might be lower since multiple optimization parameters might tradeoff against each other. 
MM Trojan will be additionally optimized to satisfy the `constraint' of bypassing a trained detector.

\begin{figure}[!h]
    \centering
    \includegraphics[scale=.4]{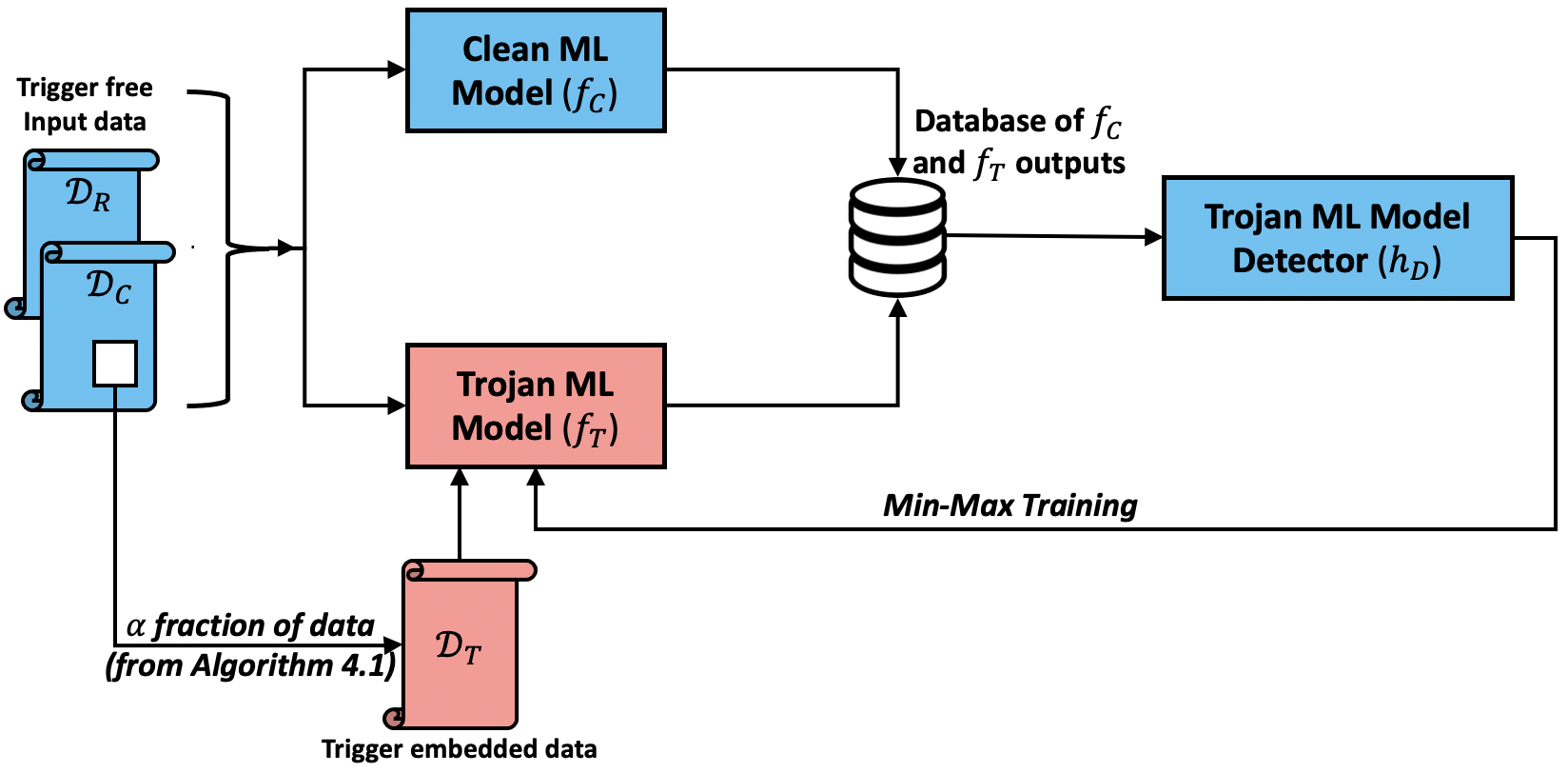}
    \caption{
    Schematic of Trojan model training. 
    Output of the min-max game returns the trained Trojan model $(f_{T}^{*})$ and the trained  detector $(h_{D}^{*})$.
    $\mathscr{D}_{C}$, $\mathscr{D}_{T}$, and $\mathscr{D}_{R}$ represent the set of clean samples, the set of samples containing a Trojan trigger, and the set of randomly generated samples, respectively. 
    The dataset $\mathscr{D}_{T} \subset \mathscr{D}_{C}$ is the fraction of samples ($\alpha$) embedded with a Trojan, as computed by the Submodular Trojan algorithm. The dataset $\mathscr{D}_{R}$ is used by the detector $(h_D)$ to distinguish between outputs of a clean model $(f_C)$ and a Trojan model $(f_T)$. Training data for $h_D$ consist of outputs of $f_C$ (with label``$0$'') and $f_T$ (with label ``$1$'').  
    }\label{fig:GAN-Train-AI-Trojan}
\end{figure}
\subsection{Analysis and Algorithm}\label{subsec:Game}
We observe that the adversary and the detection mechanism have conflicting objectives- while the detection mechanism aims to identify whether outputs corresponding to a set of random inputs are coming from a clean or a Trojan model, the adversary aims to ensure that the backdoor in the Trojan model remains undetectable. 
This interaction can be modeled as a two-player game between the Trojan model (\emph{generator}) and detector (\emph{discriminator}) \cite{goodfellow2014generative}. 
Figure~\ref{fig:GAN-Train-AI-Trojan} illustrates the overall setup. 
We denote the clean model by $f_C$, the Trojan model by $f_{T}$, and the detector by $h_{D}$. 
 
Let $p_C$ be the distribution from which each data sample $x_i\in\mathscr{D}_C$ is drawn from. 
We denote the set of (random) data samples used by $h_{D}$ to detect Trojan models by $\mathscr{D}_{R}$, which satisfies $\mathscr{D}_{R} \cap \mathscr{D}_C = \emptyset $. 
Let $\theta_D$ and $\theta_T$ denote weights of neural networks that parameterize the detector $h_{D}$ and the Trojan model $f_T$. For brevity, we define $z_{T} := f_T(x_{R}; \theta_{T})$ and $z_{C} := f_C(x_{R})$. 
Then the output of $h_{D}$ is of the form, $h_{D}(z; \theta_{D})$, representing the probability that an input $z$ is generated by the clean model $f_{C}$. The probability that an input $z$ is generated by the Trojan model can be represented as $1-h_D(z;\theta_D)$.
The interaction between the Trojan ML model and detector evolves in an iterative manner.
Each training iteration consists of two steps: (1) updating hyperparameters of $h_D$, and (2) updating hyperparameters of $f_T$.

\begin{algorithm}[!h]
\caption{Min-Max (MM) Trojan}\label{alg:minmax}
\begin{algorithmic}[1]
\Require $f_{C}(.; \theta_C), f_{T}(.;\theta_T), \mu, \sigma, itr, \mathscr{D}_T,\gamma_1, \gamma_2, \gamma_3>0$

\For {$i=1:itr$}
\State $\mathscr{D}_R\leftarrow\{x:x\sim\mathcal{N}(\mu,\sigma)\}$
\State $S\leftarrow \{[f_{T}(img;\theta_T),1],[f_{C}(img;\theta_C),0]: \forall x\in\mathscr{D}_R\}$
\State $L_1= \frac{1}{|S|}\sum_{(q,s)\in S} \frac{\partial}{\partial \theta_D} \mathcal{L}_{CE}(h_D(q;\theta_D),s)$
\State $\theta_D \gets \theta_D-\gamma_1 L_1$
\State $L_2=  \frac{1}{N'} \sum_{x_i\in \mathscr{D}_R}\frac{\partial}{\partial \theta} \mathcal{L}_{CE}( h_D(f_T(x;\theta_T);\theta_D),1) $ 
\For {$(x^k,y^k) \in \mathscr{D}_T$}

\State $L_3=L_3+ \frac{\partial}{\partial \theta} \mathcal{L}_{CE} (Y_k, f_T(x_k;\theta_T)) $
\EndFor
\State $\theta_T \gets \theta_T+ \gamma_2 L_2 -\gamma_3 L_3$
\EndFor
\end{algorithmic}
\end{algorithm}

In the first step, the Trojan model detector 
gathers outputs of clean and Trojan models for randomly generated inputs and labels them with $0$ if outputs are coming from clean models and $1$ if they are coming from a Trojan model. 
Let $\bar{p}_{T}$ and $\bar{p}_{C}$ denote probability distributions over all possible outputs $z_{T}$ and $z_{C}$, respectively. 
The detector 
then uses data of the form $\{ (z_{C},~0), (z_{T},~1) \}$ to update its parameters. 
The training procedure of the detector can be modeled as an optimization problem given by:
\begin{equation}\label{eq:max}
    \max_{\theta_{D}} \hspace{1mm} \mathbb{E}_{z_{T} \sim \bar{p}_{T}} [\log (1-h_D(z_{T};\theta_{D}))] +  \mathbb{E}_{z_{C} \sim \bar{p}_{C}} [\log (h_D(z_{C};\theta_{D}))].
\end{equation}

In the second step, the Trojan model uses outputs of the detector to update its hyper-parameters. The Trojan model maintains high accuracy on both clean and Trojan samples by jointly optimizing 
\begin{eqnarray}\label{eq:min}
    \min_{\theta_{T}} \hspace{1mm} \mathbb{E}_{z_{T} \sim \bar{p}_{T}} [\log (1-h_D(z_{T};\theta_{D}))] &+& \mathbb{E}_{x_T \sim {p}_C} [\mathcal{L}_{CE}(Y_T,f(x_T;\theta_T))]  \nonumber \\
    &+& \mathbb{E}_{x_C \sim p_C} [\mathcal{L}_{CE}(Y_C,f(x_C;\theta_T))].
\end{eqnarray}

Let $p_R(x_R)$ be the probability distribution over data samples $x_R \in \mathscr{D}_R$. Note that the term $\log (h_D(f_C(x_{R});\theta_{D}))$ is independent of $\theta_{T}$. In addition, the terms $\mathcal{L}_{CE}(Y_T,f(x_T;\theta_T))$ and $\mathcal{L}_{CE}(Y_C,f(x_C;\theta_T))$ are independent of $\theta_D$. Thus we can formulate the overall problem as a min-max optimization problem as follows:
\begin{eqnarray}\label{eq:min-max}
    \min_{\theta_{T}} \max_{\theta_{D}} \hspace{1mm} &&\mathbb{E}_{x_{R} \sim p_{R}} [\log (1-h_D(f_T(x_{R}; \theta_{T});\theta_{D}))] +  \mathbb{E}_{x_{R} \sim p_{R}} [\log (h_D(f_C(x_{R});\theta_{D}))] \\ \nonumber 
    &+& \mathbb{E}_{x_T \sim {p}_C} [\mathcal{L}_{CE}(Y_T,f(x_T;\theta_T))]  
    + \mathbb{E}_{x_C \sim p_C} [\mathcal{L}_{CE}(Y_C,f(x_C;\theta_T))]. 
\end{eqnarray}

The main result of this section establishes that the game between the detector and the Trojan model described by Equations (\ref{eq:max}) and (\ref{eq:min}) always results in the Trojan model winning the game. 
We show this by analyzing the equilibrium of this game, the output  probability distributions of the Trojan and clean models are identical. This means that the detector will not be able to distinguish between the two. 

\begin{prop}\label{prop:NE1} Let $z_{T} := f_T(x_{R}; \theta_{T})$ and $z_{C} := f_C(x_{R})$. Let $\bar{p}_{T}$ and $\bar{p}_{C}$ denote probability distributions associated with the samples $z_{T}$ and $z_{C}$, respectively. Then
the optimal solution of the min-max optimization problem in Eqn.~\eqref{eq:min-max} is such that $ \bar{p}_{T} =  \bar{p}_{C}$. 
\end{prop}
\begin{proof}
For a fixed $\theta_T$, consider the maximization part of the min-max problem in Eqn.~\eqref{eq:min-max}:
\begin{eqnarray}\label{eq:Dval}
    &&\mathbb{E}_{x_{R} \sim p_{R}} [\log (1-h_D(f_T(x_{R}; \theta_{T});\theta_{D}))] +  \mathbb{E}_{x_{R} \sim p_{R}} [\log (h_D(f_C(x_{R});\theta_{D}))] \nonumber \\
    &=&\mathbb{E}_{z_{T} \sim \bar{p}_{T}} [\log (1-h_D(z_T;\theta_{D}))] +  \mathbb{E}_{z_{C} \sim \bar{p}_{C}} [\log (h_D(z_C;\theta_{D}))] \nonumber \\
    &=&\mathbb{E}_{z \sim \bar{p}_{T}} [\log (1-h_D(z;\theta_{D}))] +  \mathbb{E}_{z \sim \bar{p}_{C}} [\log (h_D(z;\theta_{D}))] \nonumber \\
    &=& \int_{z} \bar{p}_T(z)  \log (1-h_D(z;\theta_{D})) dz + \int_{z} \bar{p}_C(z) \log (h_D(z;\theta_{D})) dz, \nonumber 
\end{eqnarray}
where the first two equalities hold by variable substitution, and the last equality holds by the definition of expectation.

To complete solving the maximization sub-problem, we leverage the fact that for any $(a, b) \in \mathbb{R}^{2} \backslash\{0,0\}$, the function $x \rightarrow a\log(1-x) + b\log(x)$ achieves its maximum in $[0,1]$ at $\frac{b}{a+b}$. 
Therefore, for a fixed $\theta_T$, $h_{D}(z, \theta_{D}^{\*}) = \frac{\bar{p}_C(z)}{\bar{p}_T(z)+\bar{p}_C(z)}$ when $\theta_T$ is given.

With the above value of $h_{D}(z, \theta_{D}^{\*})$, consider the minimization part of Eqn.~\eqref{eq:min-max}. We have 

\begin{eqnarray}\label{eq:Gval1}
    && \int_{z} \bar{p}_T(z)  \log \Big(1-\frac{\bar{p}_C(z)}{\bar{p}_T(z)+\bar{p}_C(z)}\Big) dz + \int_{z} \bar{p}_C(z) \log \Big(\frac{\bar{p}_C(z)}{\bar{p}_T(z)+\bar{p}_C(z)}\Big) dz \nonumber \\
    &=& \int_{z} \bar{p}_T(z)  \log \Big(\frac{\bar{p}_T(z)}{\bar{p}_T(z)+\bar{p}_C(z)}\Big) dz + \int_{z} \bar{p}_C(z) \log \Big(\frac{\bar{p}_C(z)}{\bar{p}_T(z)+\bar{p}_C(z)}\Big) dz. \label{Eq:IntermediateStep}
\end{eqnarray}
The two terms in Eqn. (\ref{Eq:IntermediateStep}) can be interpreted as KL divergences between probability distributions. 
The KL divergence between two probability distributions $p$ and $q$ is defined by $KL(p||q) := \int_{-\infty}^{+\infty}p(x)\log\Big(\frac{p(x)}{q(x)}\Big)dx$. Thus, $\int_{-\infty}^{+\infty}p(x)\log\Big(\frac{p(x)}{p(x)+q(x)}\Big)dx = KL(p||(p+q)/2)$. Here, $(p+q)/2$ is used to ensure validity of the probability distribution such that $\int_{x} (p(x)+q(x))/2) = 1$. 
Substituting in Eqn. (\ref{Eq:IntermediateStep}), we obtain
\begin{align}\label{eq:Gval1}
& \int_{z} \bar{p}_T(z)  \log \Big(\frac{\bar{p}_T(z)}{\bar{p}_T(z)+\bar{p}_C(z)}\Big) dz + \int_{z} \bar{p}_C(z) \log \Big(\frac{\bar{p}_C(z)}{\bar{p}_T(z)+\bar{p}_C(z)}\Big) dz \nonumber \\
    =~& \text{KL}\Big(\bar{p}_{T} || \frac{\bar{p}_{T} + \bar{p}_{C}}{2}\Big) + \text{KL}\Big(\bar{p}_{C} || \frac{\bar{p}_{T} + \bar{p}_{C}}{2}\Big). \nonumber
\end{align}
The minimum is achieved when $\bar{p}_T =\frac{\bar{p}_{T} + \bar{p}_{C}}{2}$, or equivalently, when $\bar{p}_T = \bar{p}_C$. 
\end{proof}
Proposition \ref{prop:NE1} shows that the detector will not be able to distinguish between probability distributions of the Trojan and clean model outputs since they are identical (i.e., $\bar{p}_T = \bar{p}_C)$.  Hence the output of the detector $\theta_T$, $h_{D}(z, \theta_{D}^{\*}) = \frac{\bar{p}_C(z)}{\
\bar{p}_T(z)+\bar{p}_C(z)} = 0.5$ for both clean and Trojan samples. This is analogous to decision making based on the toss of a fair coin, thus rendering the detector obsolete. 
Therefore, the adversary wins the min-max game by always evading detection.

The \emph{Min-Max (MM) Trojan} Algorithm (Algorithm \ref{alg:minmax}) describes the min-max optimization procedure. 
In the $\min$ step, we first generate a set of arbitrary inputs (images,
in our case) and provide them to both non-Trojan and Trojan models. 
The outputs from these models are used to train the discriminator. 
In the $\max$ step, we update parameters of the Trojan model to maximize the loss of the discriminator by generating outputs that are similar to outputs of the non-Trojan model for the arbitrarily generated inputs. 
\section{Evaluation}
\label{sec:evaluation}
This section introduces our simulation setup and details the results of our evaluations of the Submodular Trojan and MM-Trojan algorithms described in Sections \ref{sec:Submod} and \ref{sec:GameModel}. 
The objective of the adversary is to (i) minimize the fraction of samples into which a Trojan needs to be embedded, and (ii) evade detection by an adaptive instance-based detection mechanism, while simultaneously ensuring high accuracy of classifying clean samples correctly and that of classifying Trojan samples to the target class. 

\subsection{Datasets and Experiment Setup} 

We describe the datasets and architectures of neural networks used in our experiments. 
In each case, a white square of size $4 \times 4$ is used as the trigger.\\
\noindent{\it MNIST:} 
This dataset contains $60000$ images of hand-written digits ($\{0,1,\cdots,,9\}$), of which $50000$ are used for training and $10000$ for testing. The size of each image is $28 \times 28$. 
The DNNs used to learn a classifier for the MNIST dataset consists of two convolutional layers, each containing $5$ kernels, and channel sizes of $16$ and $32$ respectively. 
This is followed by maxpooling and fully connected layers of size $512$. 

\noindent{\it CIFAR-10:} 
This dataset contains $60000$ images of 10 different objects (e.g., car, horse, etc.), of which $50000$ are used for training and $10000$ for testing. The size of each image is $32 \times 32$. 
The DNNs used to learn a classifier for the CIFAR-10 dataset consists of four convolutional layers, each containing $3$ kernels, and channel sizes of $32$, $32$, $64$, and $64$ respectively. 
This is followed by maxpooling and two fully connected layers of size $256$. 

\noindent{\it EuroSAT:} This dataset is based on Sentinel-2 satellite images of regions for land use and land cover classification published  
in~\cite{helber2019eurosat}. 
The dataset contains 27000 images with resolution of $64\times64$. We use $80\%$ of these images for training and $20\%$ for test.

\noindent{\it Detector:} 
The detector uses a network with one (fully connected) hidden layer of size $20$ and a softmax activation in the last layer. 
The output of the detector is a vector whose first element indicates the probability that the input to the the detector was from a clean model, and whose second element indicates that the input came from a Trojan model. 

\subsection{Metrics and Baselines}
In the following, we present the evaluation metrics and baselines.

\underline{\emph{ Accuracy on Clean Samples (Acc-C):}} This metric measures the accuracy of classification of clean samples by an ML model at test time, and is defined as:
\begin{equation}
    \text{Acc-C}= \frac{1}{|\mathscr{D}_{test}|} \sum_{(x_i,Y_i) \in\mathscr{D}_{test}} \mathbb{I}(\text{argmax}(Y_i) = \text{argmax}(f_T( x_i;\theta_T))), 
\end{equation}
where $|\mathscr{D}_{test}|$ is the size of a test set that contains only clean samples, argmax returns the class of $x_i$ with highest probability, and $\mathbb{I}(\cdot)$ returns $1$ if its argument is true, and $0$ otherwise.

\underline{\emph{Accuracy on Trojan Samples (Acc-T):}} 
The metric Acc-T is computed by inserting a trigger into all samples in the test set, $\mathscr{D}_{test}$, and counting the number of samples that are classified by the model to the target class, denoted as $\text{argmax}(Y_T)$. 
\begin{equation}
    \text{Acc-T}= \frac{1}{|\mathscr{D}_{test}|} \sum_{x_i \in\mathscr{D}_{test}} \mathbb{I}(\text{argmax}(Y_T) = \text{argmax}(f_T(x_i \times (1-\Delta)+ \delta \times \Delta;\theta_T))). 
\end{equation}
We note that the above quantity is sometimes termed an \emph{attack success rate (ASR)} in the machine learning literature~\cite{gao2020backdoor}. 
However, we use the notation Acc-T to avoid ambiguity with the objective of the adversary to evade detection in our paper. 

\underline{\emph{Baseline}:} 
The MM Trojan algorithm in Section \ref{sec:GameModel} which was shown to be able to bypass Trojan model detection might result in a decrease in accuracy of classification of clean samples. 
We evaluate the performance of MM Trojan by comparing its Acc-C and Acc-T with Clean and Baseline Trojan models. 
The Submodular Trojan algorithm in Section \ref{sec:Submod} computes the optimal fraction $\alpha$ of the dataset into which a Trojan must be embedded in a constructive manner. 
We start with an initial choice of this parameter, and gradually increase it until there is no discernible improvement in adversary performance measured in terms of a decrease in the value of a loss function. 

\begin{figure*}
\centering
\begin{tabular}{ c c c c}
    \includegraphics[trim={2cm 2cm 6cm 2cm}, scale=0.075]{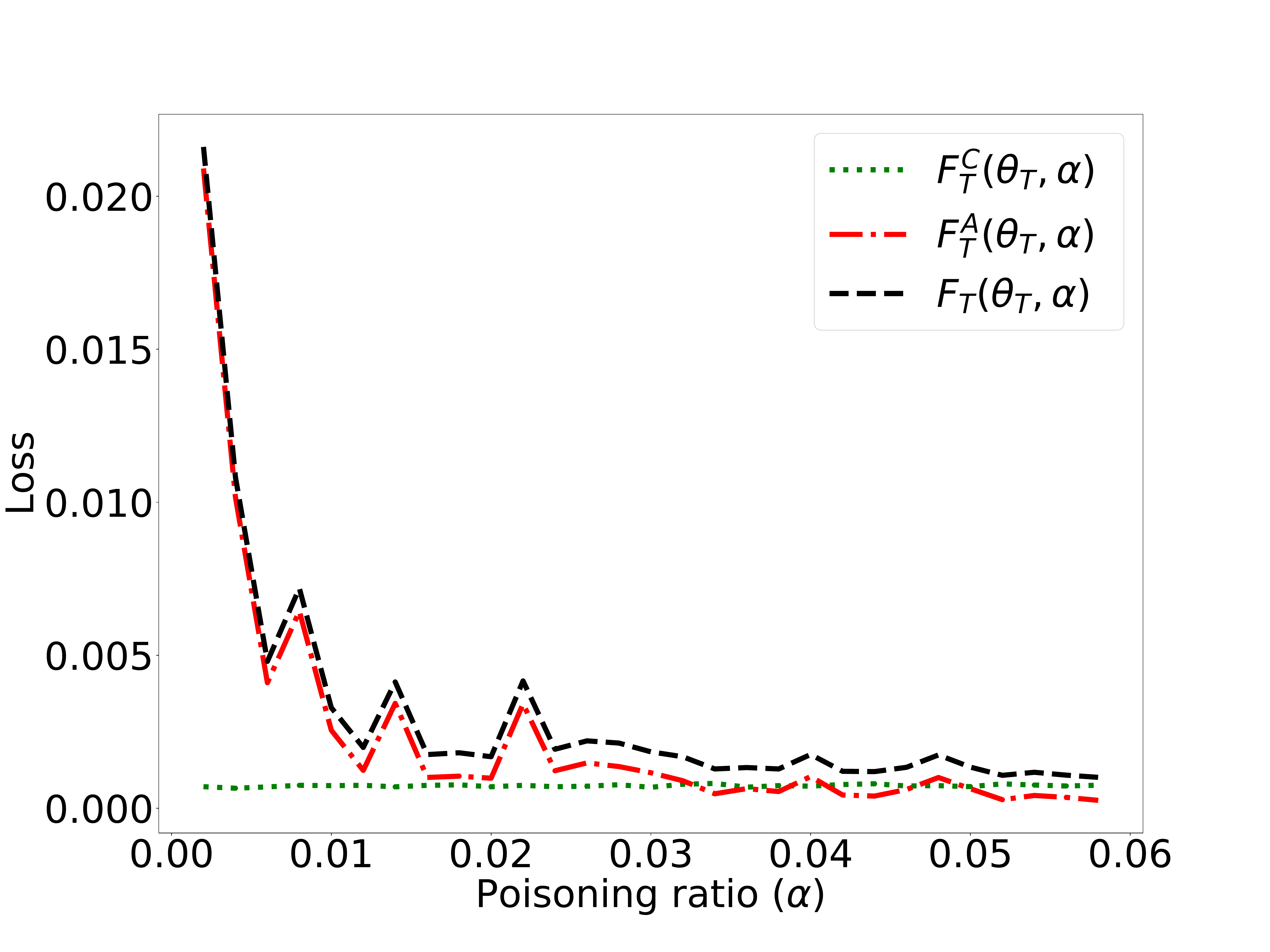} & 
    \includegraphics[trim={2cm 2cm 6cm 2cm}, scale=0.075]{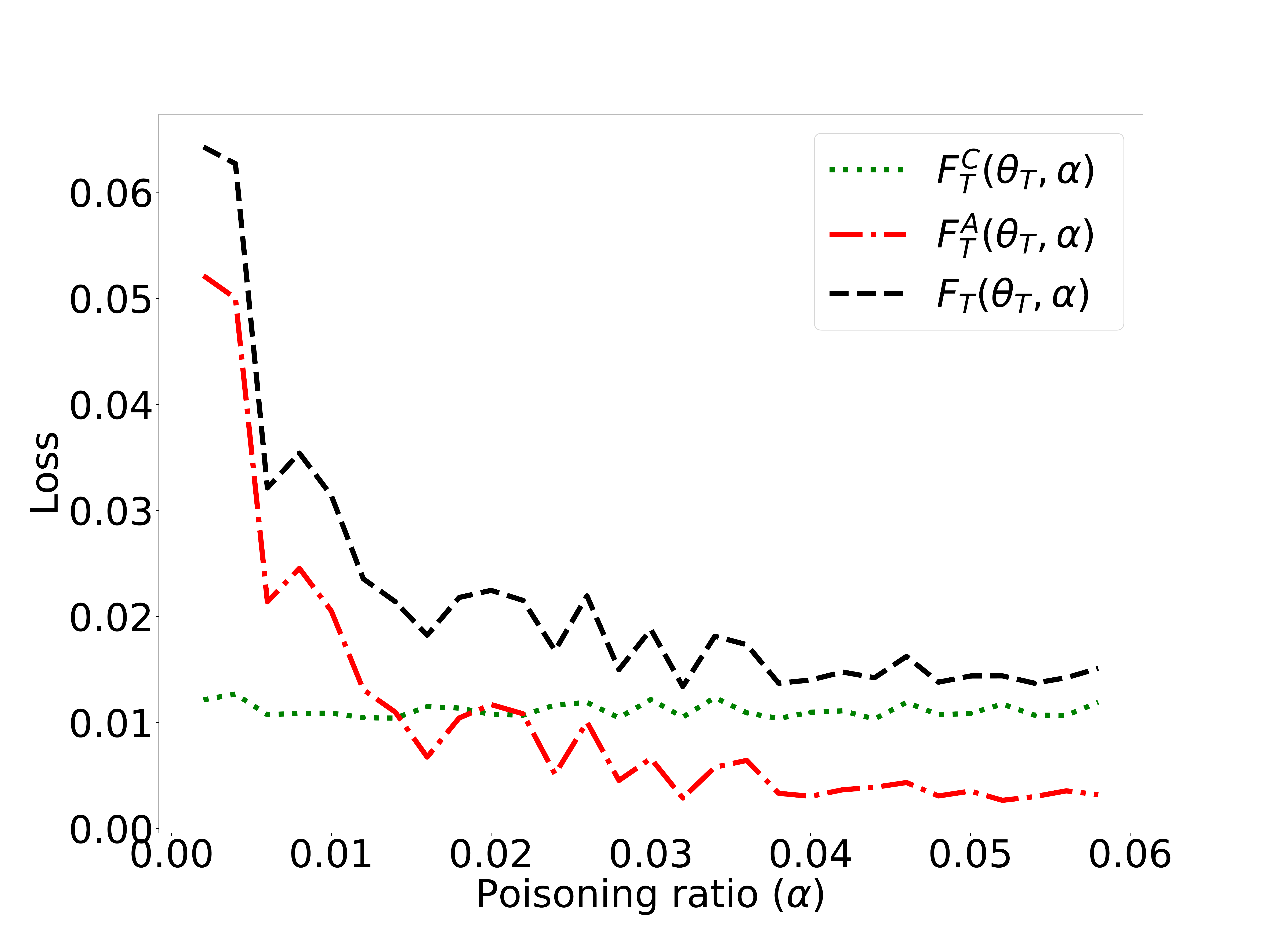} &
    \includegraphics[trim={2cm 2cm 6cm 2cm}, scale=0.075]{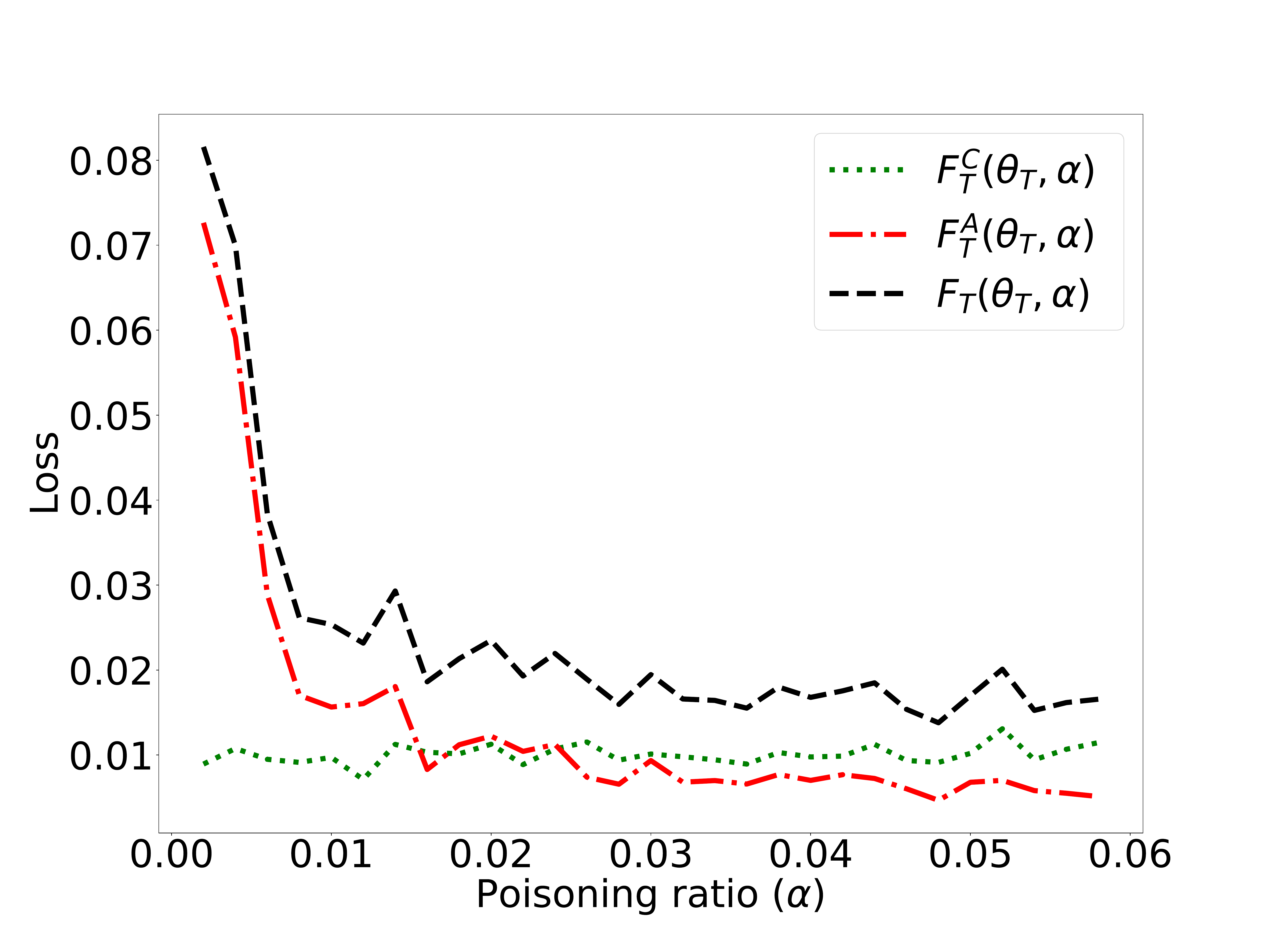}  \\
    
      MNIST& CIFAR-10& EuroSAT\\
         
\end{tabular}
\caption{
The figure presents the loss incurred during training as functions of the fraction of samples embedded with the Trojan (poisoning ratio, $\alpha$). The green curves represent the loss $F_T^C(\theta_T,\alpha)$ as a function of poisoning ratio $\alpha$ when training over the clean samples, and the red curves represent the loss $F_T^A(\theta_T,\alpha)$ as a function of poisoning ratio $\alpha$ due to embedding Trojans by the adversary. The black curves present the total loss $F_T(\theta_T,\alpha)=F_T^A(\theta_T,\alpha)+F_T^C(\theta_T,\alpha)$ as defined in Eqn. \eqref{eq:alpha_prob}.
}
\label{fig:subloss}
\end{figure*}

\subsection{Evaluation of Submodular Trojan Algorithm (Algorithm \ref{alg:greedy})}

Figures~\ref{fig:subloss} and \ref{fig:subacc} present experimental evaluations of the Submodular Trojan Algorithm on the MNIST, CIFAR-10, and EuroSAT datasets. 
In all experiments, we initialize the fraction of the dataset $\alpha$ into which a Trojan has been embedded to $0.002$. 
In each iteration, we train a model to minimize the loss function in Eqn. (\ref{eq:alpha_prob}). 
This loss function consists of two terms- one over the clean samples, and another over Trojan samples. 
We repeat the procedure by increasing the value of $\alpha$ (Line 6 of Algorithm \ref{alg:greedy}). 
The procedure terminates when increasing $\alpha$ does not result in a significant decrease in the value of the loss function. 

Figure~\ref{fig:subloss} shows the change in loss values as the value of $\alpha$ increases. 
 In each case, we observe that the loss term corresponding to Trojan samples (red curve) initially decreases as $\alpha$ increases. 
 However, after exceeding a threshold, the decrease in the total loss value (black curve) is insignificant. 
 The loss term corresponding to clean samples (green curve) might increase as the fraction of clean samples decreases with increase in $\alpha$. 
 This demonstrates that the adversary has little advantage in poisoning a fraction of samples larger than the threshold. 
 For example, $\alpha^* \approx 0.03$ for MNIST. 
Figure~\ref{fig:subacc} shows Acc-C (green curve) and Acc-T (red curve) as $\alpha$ is increased. 
 For small values of $\alpha$, the Acc-T is small but this metric improves as $\alpha$ is increased. 
 The improvement becomes less significant after a threshold of $\alpha$ is reached. 
 Simultaneously, the value of Acc-C decreases slightly, which indicates that the model might become less useful to a user, leading them to discard it. Thus poisoning a large fraction of samples will not be advantageous to the adversary. 
 
 The outcomes of our experiments is consistent with the theoretical analysis presented in Section \ref{sec:Submod}. 
 Further, the constructive manner in which the optimal $\alpha$ is determined reveals that the adversary will need to embed a Trojan into a very small fraction of the dataset ($\alpha^* \approx 0.03$ for MNIST, $\alpha^* \approx 0.04$ for CIFAR-10 and EuroSAT). 
 This is significantly smaller than the $10\%$ of training data that consisted of Trojan samples in \cite{guo2022aeva}.

\begin{figure*}
\centering
\begin{tabular}{ c c c c}

     \includegraphics[ scale=0.1]{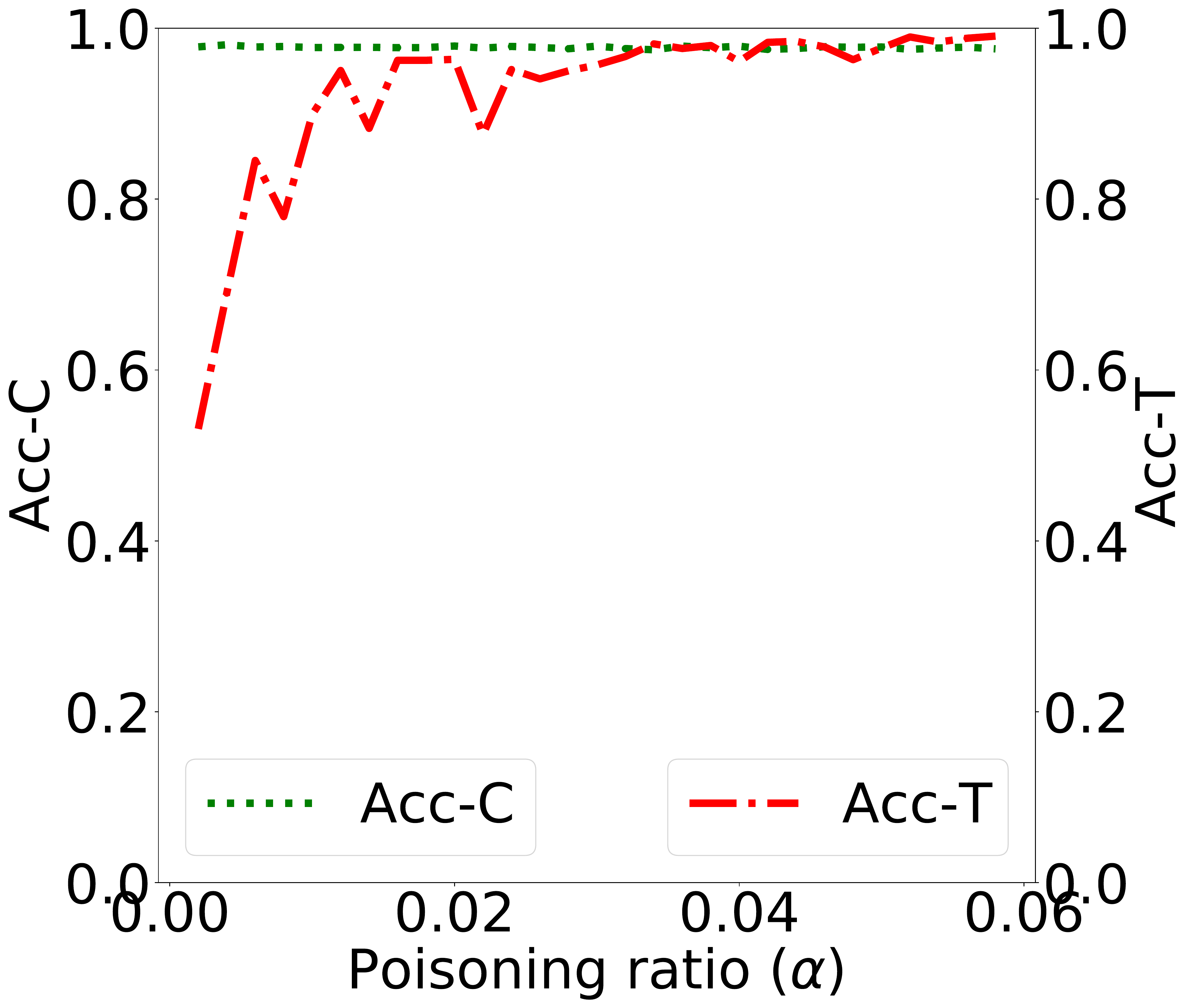} & 
    \includegraphics[scale=0.1]{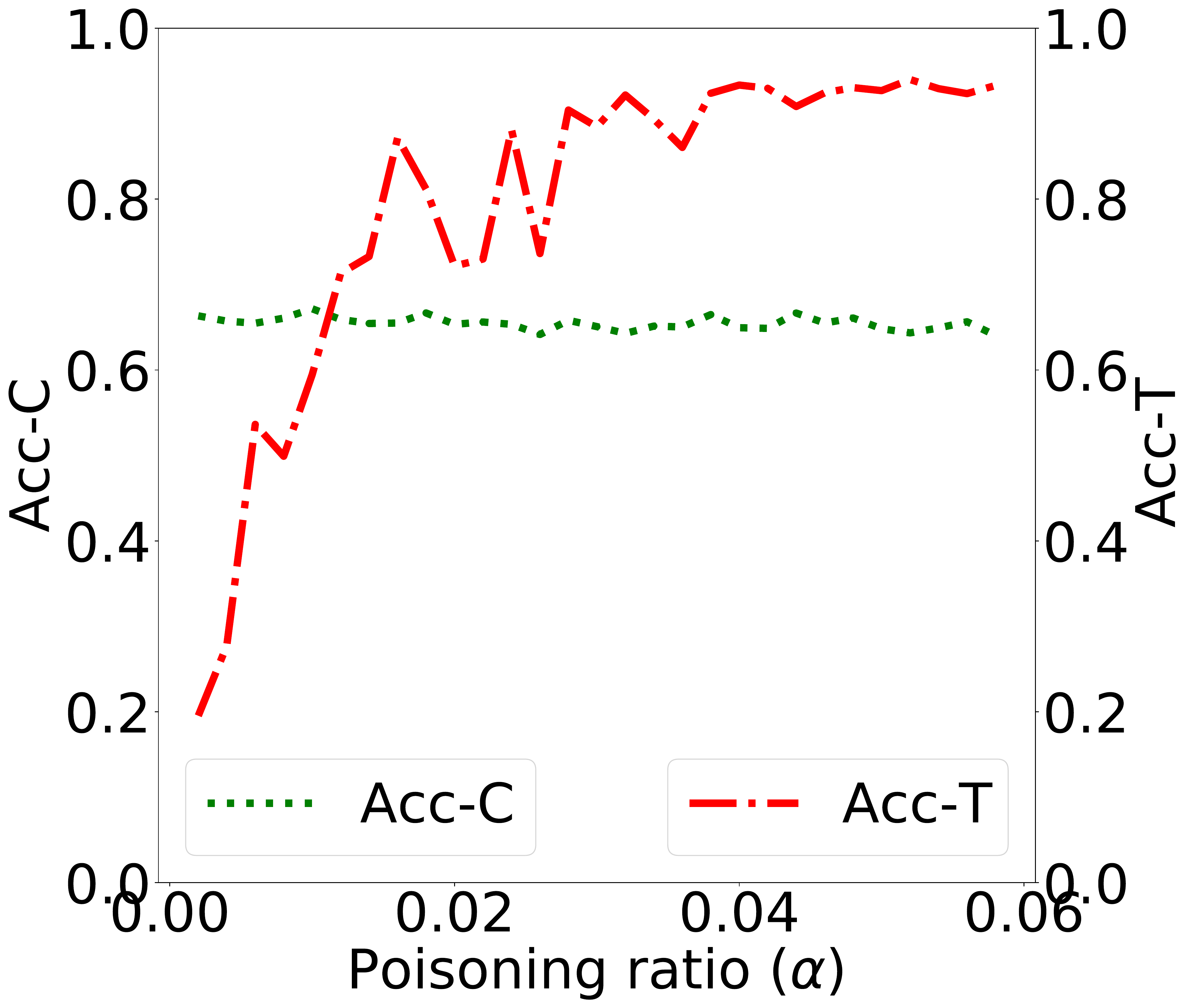} &
    \includegraphics[scale=0.1]{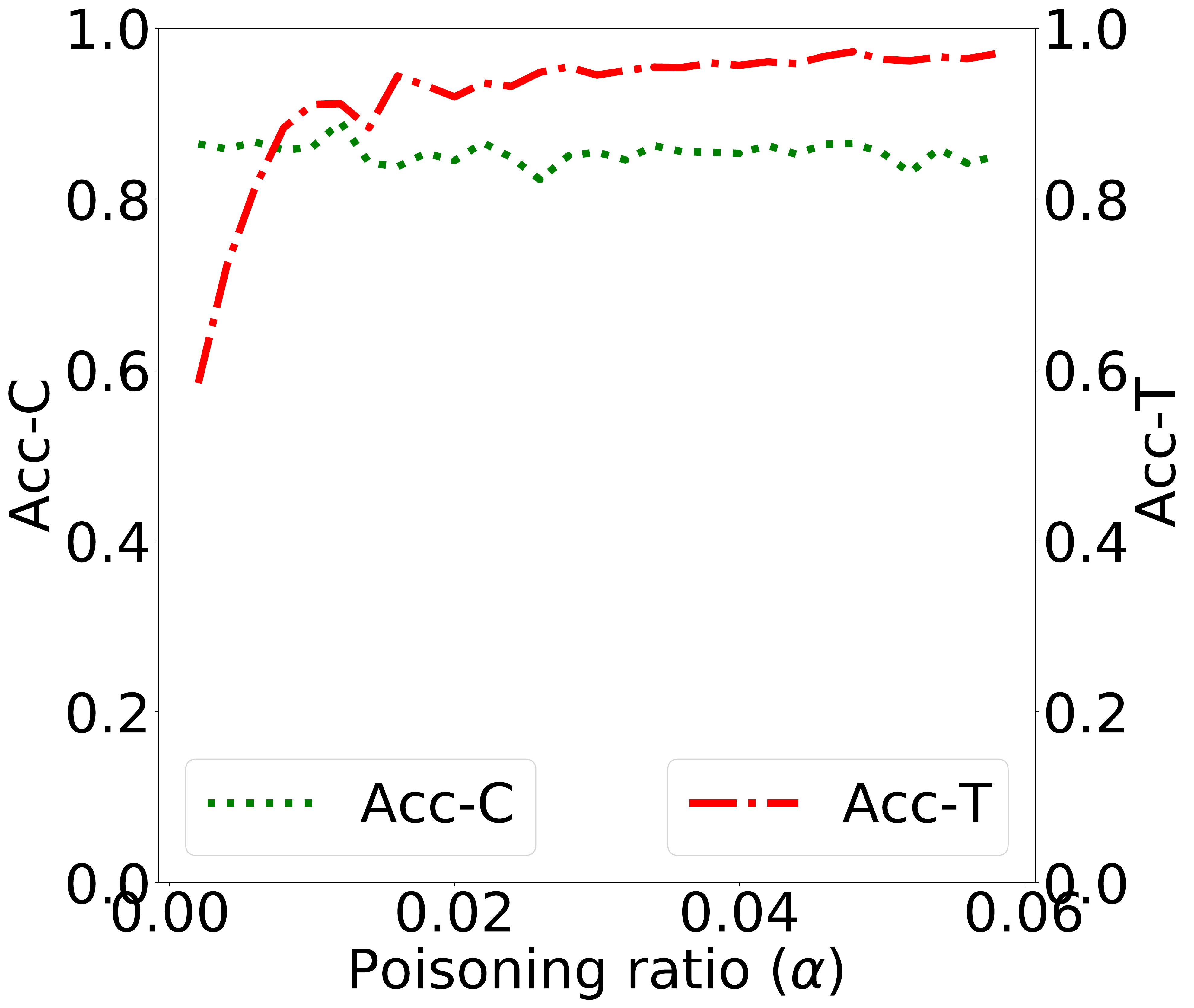}  \\
    
 \\

      MNIST& CIFAR-10& EuroSAT
\end{tabular}
\caption{The figure presents classification accuracy over the clean samples (Acc-C, green curve) and Trojan embedded samples (Acc-T, red curve) when using the MM Trojan model trained with different poisoning ratio $\alpha$. The improvement of Acc-T becomes less significant beyond a threshold value of $\alpha$, while Acc-C may reduce as $\alpha$ increases (which might lead a user to discard the model).
}
\label{fig:subacc}
\end{figure*}

\subsection{Evaluation of MM Trojan Algorithm (Algorithm \ref{alg:minmax})}

Table~\ref{tab:acc} and Figure~\ref{fig:confacc} present experimental evaluations of the MM Trojan algorithm on the MNIST, CIFAR-10, and EuroSAT datasets. 

\begin{table}[]
    \centering
    \begin{tabular}{|c|c|c|c|c|}
     \toprule
      Dataset & \shortstack{Model} & \shortstack{Acc-C} & \shortstack{ Acc-T}& \shortstack {Evasion of Detection}\\ \hline
       \multirow{3}{*}{MNIST} & Clean &  $99.12\%$ & 0& NA \\  \cline{2-5}
       & Baseline Trojan& $98.97\%$ &  $99.96\%$ &$0\%$ \\  \cline{2-5}
       & MM Trojan&  $95.36\%$  &  $95.30\%$ &$\mathbf{100\%}$ \\ \hline
        \multirow{3}{*}{CIFAR10} & Clean & $70.10\%$ & 0 & NA \\ \cline{2-5}
       & Baseline Trojan&  $70.15\%$& $93.53\%$ &$\mathbf{100\%}$ \\  \cline{2-5}
       & MM Trojan&  $69.35\%$ &  $94.90\%$ &$\mathbf{100\%}$\\ \hline
       \multirow{3}{*}{EuroSAT} & Clean & $86.09\%$ &  0 &NA\\  \cline{2-5}
       & Baseline Trojan& $83.69\%$ & $95.17\%$ &$0\%$\\  \cline{2-5}
       & MM Trojan& $86.46\%$ &  $97.13$ &$\mathbf{100\%}$\\
       \bottomrule
    \end{tabular}
    \caption{The table presents the classification accuracy on clean samples (Acc-C) and Trojan samples (Acc-T) when applying Clean, Baseline Trojan, and MM-Trojan models. The value of Acc-T for Clean models is $0$ since they do not contain a trigger. MM Trojan successfully preserves high values of Acc-C and Acc-T compared to the Baseline Trojan. In addition, MM-Trojan is also able to bypass detection $100\%$ of the time, while the Baseline Trojan does not for MNIST and EuroSAT. The Baseline Trojan bypasses detection on CIFAR-10, possibly because a detector designed to be optimal against MM Trojan might be suboptimal against the Baseline.}
    \label{tab:acc}
\end{table}
\vspace*{-5mm}
The accuracy on clean (Acc-C) and Trojan (Acc-T) samples for the Clean, Baseline Trojan, and MM Trojan models are shown in Table~\ref{tab:acc}. 
The value of Acc-C is high for Clean models. 
The absence of Trojan samples in the Clean model implies that Acc-T will be zero. 
The Baseline Trojan provides similar Acc-C to Clean, since it is trained using clean and Trojan samples. 
Acc-C and Acc-T of MM Trojan is comparable to the Baseline Trojan for all three datasets. 
Differences (drop or rise) in Acc-C or Acc-T for MM Trojan will be contingent on the choices of values of learning rates $\gamma_2$ and $\gamma_3$ in Algorithm~\ref{alg:minmax}. 
For example, while a low value of $\gamma_2$ and high value of $\gamma_3$ cause Acc-C and Acc-T to increase, it will also make it more likely that the detection probability will increase. 
The last column of the table shows the probability that an adversary will be able to bypass detection. 
We observe that an adversary trained on our MM Trojan model is able to always evade detection ($100\%$) for all considered datasets, while the Baseline Trojan cannot for MNIST and EuroSAT ($0\%$).  
Note that evading detection is not defined ($NA$) for Clean models. In the remainder of this subsection, we will detail the evasion of detection.
\begin{figure*}
\centering
\begin{tabular}{c c c c c}
    \rotatebox{90}{$\:\:\:\:\:\:\:\:\:\:\:\:$}&
    \includegraphics[trim={2cm 2cm 4cm 2cm}, scale=0.09]{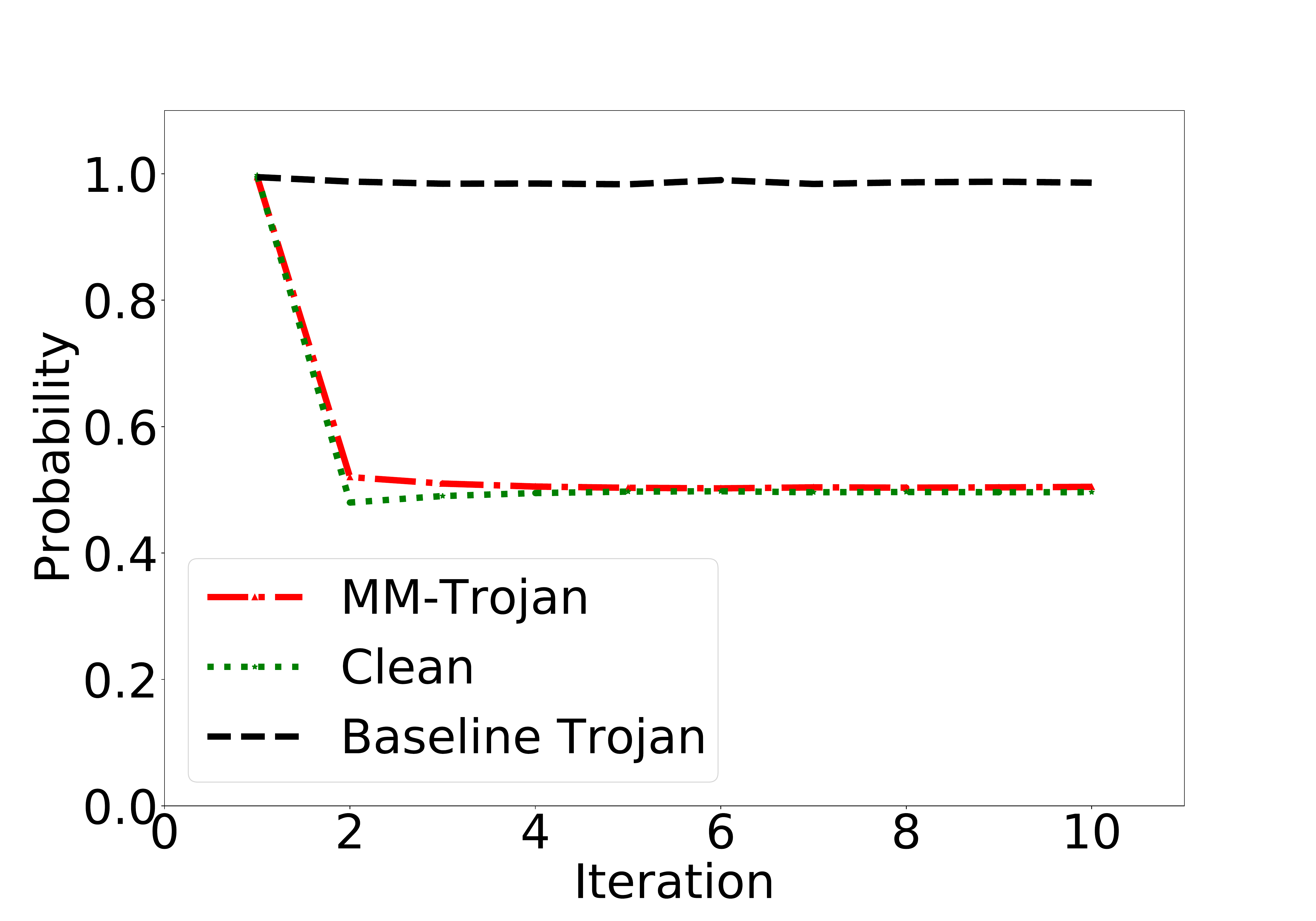} &
    \includegraphics[trim={2cm 2cm 4cm 2cm}, scale=0.09]{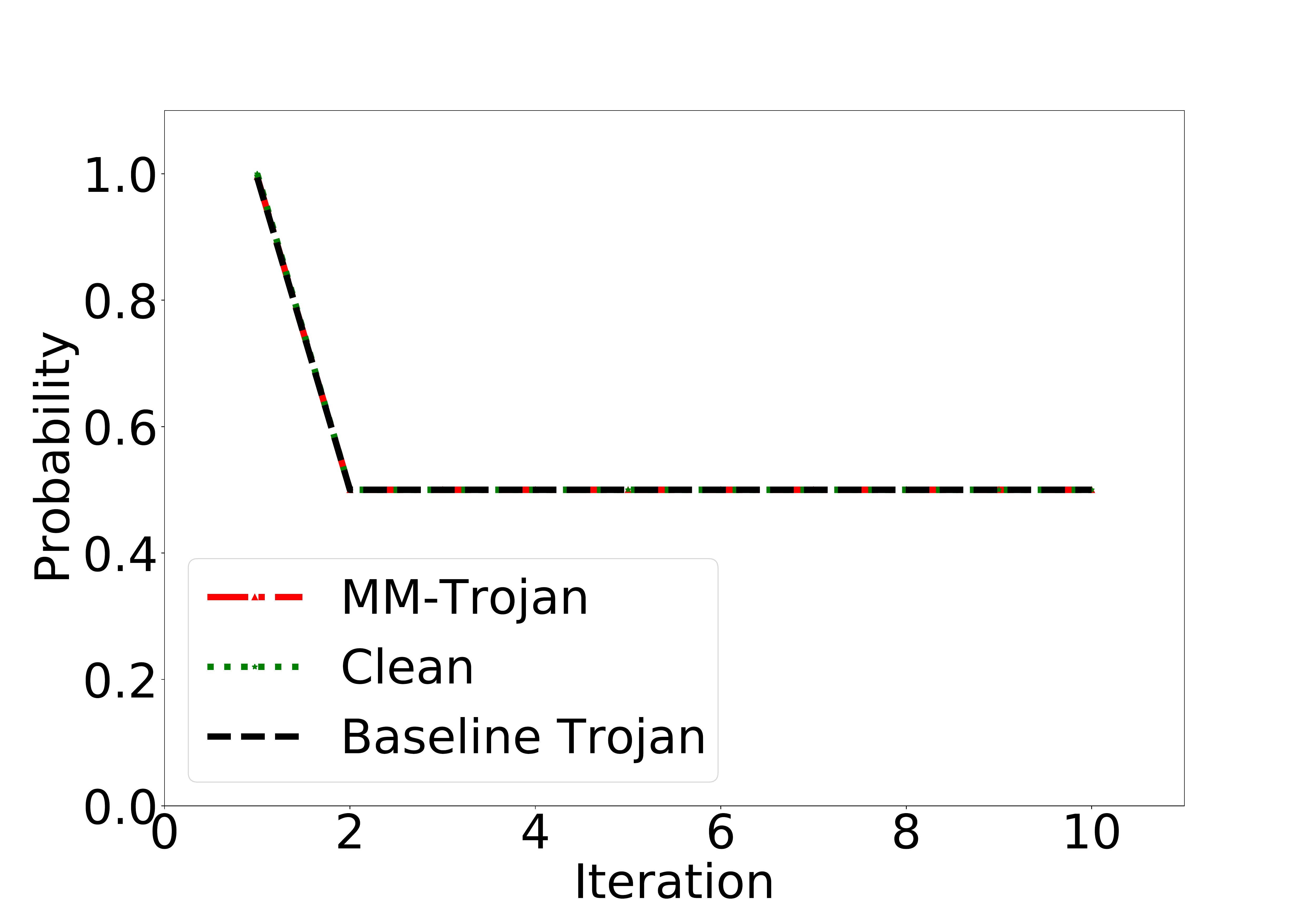} &
    \includegraphics[trim={2cm 2cm 4cm 2cm}, scale=0.09]{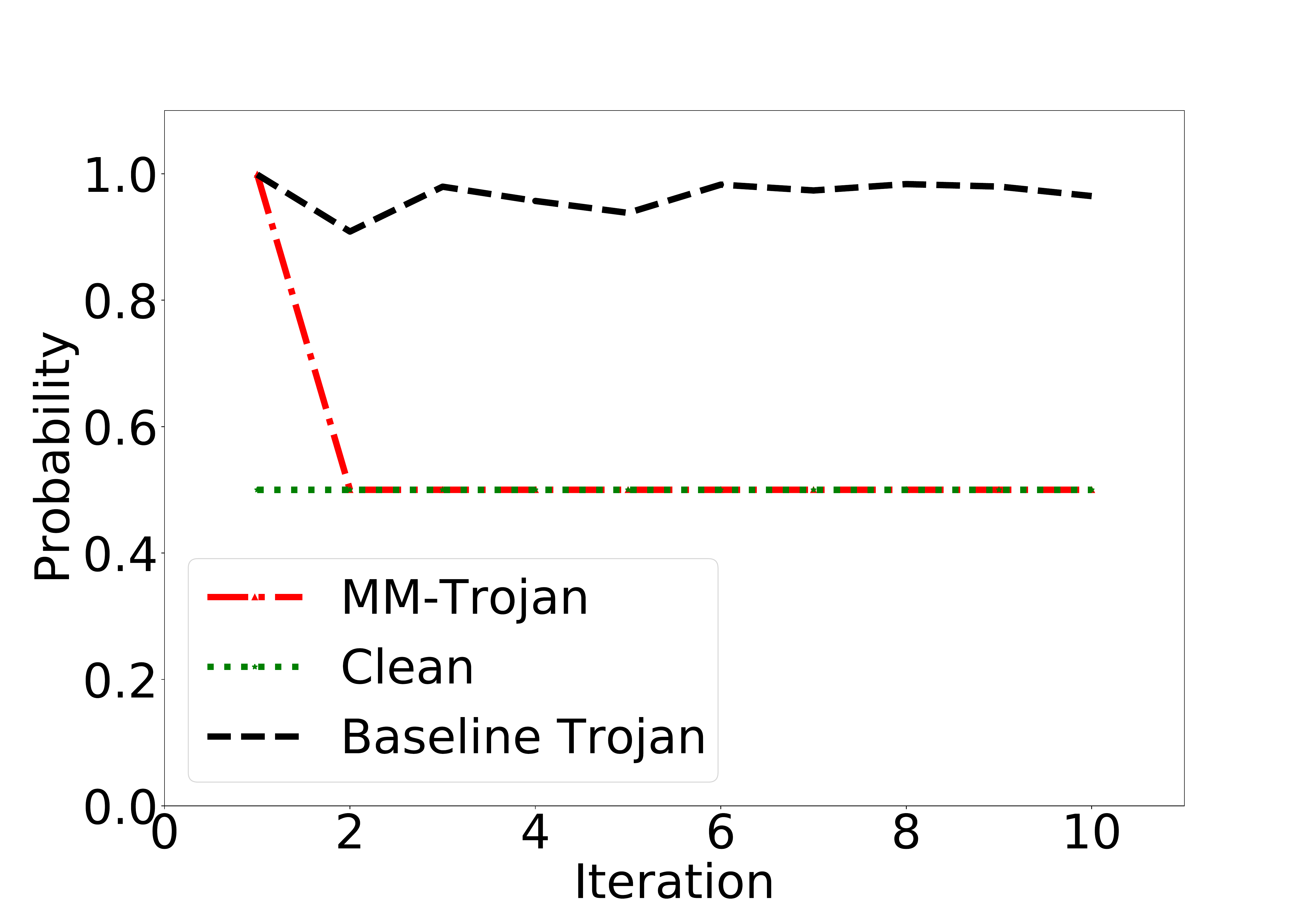}  \\
   
     \\
     & MNIST& CIFAR-10& EuroSAT\\
\end{tabular}
\caption{The figure presents the probability that the detector identifies a model as Clean or Trojan when provided with outputs generated from these models. The detector is unable to distinguish between clean and Trojan for MM Trojan models, as shown by the probability of $0.5$. In contrast, the detector is able to identify Trojan models with probability $1$ for the Baseline Trojan for MNIST and EuroSAT. The detector is unable to distinguish between clean and Trojan models for the Baseline Trojan on CIFAR-10, possibly because a detector designed to be optimal against MM Trojan might be suboptimal against the Baseline.}
\label{fig:confacc}
\end{figure*}

We provide the outputs of Clean, Baseline Trojan, and MM Trojan to the detector. 
Figure~\ref{fig:confacc} shows the probability that the detector can identify if the output is coming from a clean or a Trojan model. 
We observe that the detector is unable to distinguish outputs in the cases of Clean or MM Trojan, as illustrated by the fact that the probability value is $0.5$. 
The detector correctly classifies the Baseline Trojan for MNIST and EuroSAT. 
However, it fails for CIFAR-10. 
A possible reason for this is the fact that the detector is designed to be optimal against MM Trojan, and this might not be the case for the Baseline Trojan. 
We will examine and investigate strategies to train Trojan models to bypass a broad class of detection mechanisms in future work. 

Results from Table \ref{tab:acc} and Fig. \ref{fig:confacc} show that the performance of MM Trojan is the best in the sense that (i) the detector is bypassed for all three datasets, (ii) Acc-C is comparable to Clean and Baseline Trojan models, and (iii) Acc-T is comparable to the Baseline Trojan.
\section{Discussion}\label{sec:Discussion}
In this section, we discuss potential impacts of relaxing certain assumptions in our threat models,  analyses, and experiments in the previous sections. We also highlight some open questions that are promising research directions. 

\underline{\emph{Information available to adversary:}}  
We assume that the adversary has knowledge of statistical parameters (e.g., mean, standard deviation) of samples used by the detector to detect a Trojan model. 
A promising research direction is to characterize conditions under which the Trojan model will win the two-player game with the detector when the adversary has limited statistical information about samples used by the detector. 

\underline{\emph{Learning universal detectors:}}
In Sec. 6.4, we saw that while the detector was able to correctly distinguish between outputs of a clean and Baseline Trojan model for the MNIST and EuroSAT datasets, this was not the case for CIFAR-10. 
Our conjecture is that a detector trained to be optimal against MM Trojan might be far from optimal against other types of Trojan models. 
An interesting direction of future research is to develop techniques to train Trojan models to bypass different types of detectors. 
One could also investigate learning detectors that will be effective against a broad class of Trojan training methodologies. 

\underline{\emph{Access to intermediate DNN layers:}}  
In this paper, we assume that only black-box access to ML models is available.
The performance of the detector can be improved by enabling access to outputs of intermediate layers of DNNs~\cite{kolouri2020universal}. 
One strategy is to compare outputs of intermediate layers of a Trojan model to outputs of a clean model, assuming that the output from the last layer of the Trojan model is similar to that of a clean model for the same input distribution. 
Quantifying impacts of access to intermediate layers on Acc-C and Acc-T when the adversary uses the MM Trojan Algorithm \ref{alg:minmax} to train undetectable Trojan models is a promising research direction. 

\underline{\emph{Ranking samples:}}  
In this paper, we assume that all samples are equivalent in the sense that embedding a trigger into any sample in the set of size $\lfloor \alpha N \rfloor$ will achieve the same loss value. 
A natural extension of our approach to select the optimal subset of samples $\alpha$ in Sec. \ref{sec:Submod} is to rank samples based on their influence. 
Such a ranking can then be used by the adversary to identify `critical' samples into which a trigger can be embedded in order to optimize its objective. 
We will also examine strategies to generate \emph{highly ranked} samples to maximally affect loss values. 

\underline{\emph{Establishing tighter bounds:}} 
The objective function in Eqn \eqref{eq:alpha_prob} of the adversary may not be always be continuous in $\alpha$. 
A potential strategy to help establish tighter bounds for the Submodular Trojan Algorithm \ref{alg:greedy} is to find better approximations of Eqn \eqref{eq:alpha_prob}. 
To accomplish this, we can develop a variant of the supermodular property that captures the impact of $\alpha$ due to the $\frac{1}{\alpha N}$ 
term, as well as the discrete jump induced by changes in the limits of the summation in Eqn. \eqref{eq:alpha_prob}. 
Furthermore, the Submodular Trojan algorithm could also leverage the convexity of $\Bar{F}_T(\theta_T,\alpha)$ to yield tighter optimality guarantees.

\section{Conclusion}
\label{sec:Conclusion}

This paper presented an analytical characterization of adversarial capability and strategic interactions between an adversary and detection mechanism to identify Trojan machine learning models. 
The adversary capability was modeled in terms of the fraction $\alpha$ of the training input dataset that could be embedded with a Trojan trigger. 
We showed that the loss function had a submodular structure, which led to the design of a computationally efficient algorithm to determine this fraction with provable bounds on optimality. 
We presented a \emph{Submodular Trojan} algorithm to constructively determine the minimal fraction of samples to inject a Trojan trigger. 
To evade detection of the Trojaned model by potential detectors, we modeled the strategic interactions between the adversary and Trojan detection mechanism as a two-player game. We established that the adversary wins the game with probability one, thus bypassing detection. 
We proved this by showing that output probability distributions of a Trojan model and a clean model were identical when following our proposed \emph{Min-Max (MM) Trojan} algorithm. 
We carried out extensive evaluations of our algorithms 
on MNIST, CIFAR-10, and EuroSAT datasets. 
The results validated our theoretical analysis by demonstrating that (i)
with \emph{Submodular Trojan} algorithm, the adversary needed to embed a Trojan trigger into a very small fraction ($\approx 5\%$) of samples to achieve high accuracy on both Trojan and clean samples (e.g., $97\%$ and $85\%$ for EuroSAT), and 
(ii) the \emph{MM Trojan} algorithm yielded a trained Trojan model that would evade detection with probability one. 

A promising direction of future research is to design and develop techniques to train a Trojan model to bypass a detector that has been trained to be optimal against a different Trojan model. 
A second research question is to examine the effect on the value of $\alpha$ when samples are ranked, and when the information available to the adversary about parameters of the detector will vary.

\bibliographystyle{splncs04}      
\bibliography{GameSec2022-references}

\begin{thebibliography}{10}
\providecommand{\url}[1]{\texttt{#1}}
\providecommand{\urlprefix}{URL }
\providecommand{\doi}[1]{https://doi.org/#1}

\bibitem{AWS}
Amazon, {M}achine learning at {AWS}, 2018.

\bibitem{BigML}
Big{ML} {I}nc. {Big}{ML}, 2018

\bibitem{Caffe}
Caffe, {C}affe {M}odel {Z}oo, 2018.

\bibitem{bian2017continuous}
Bian, A., Levy, K., Krause, A., Buhmann, J.M.: Continuous {DR}-submodular
  maximization: Structure and algorithms. Advances in {N}eural {I}nformation
  {P}rocessing {S}ystems  \textbf{30} (2017)

\bibitem{BiaMirBuh-16}
Bian, A.A., Mirzasoleiman, B., Buhmann, J.M., Krause, A.: Guaranteed non-convex
  optimization: Submodular maximization over continuous domains. International
  Conference on Artificial Intelligence and Statistics pp. 111--120 (2017)

\bibitem{bojarski2016end}
Bojarski, M., Del~Testa, D., Dworakowski, D., Firner, B., Flepp, B., Goyal, P.,
  Jackel, L.D., Monfort, M., Muller, U., Zhang, J., Zhang, X., Zhao, J., Zieba,
  K.: End to end learning for self-driving cars. arXiv preprint
  arXiv:1604.07316  (2016)

\bibitem{chan2019poison}
Chan, A., Ong, Y.S.: Poison as a cure: Detecting \& neutralizing variable-sized
  backdoor attacks in deep neural networks. arXiv preprint arXiv:1911.08040
  (2019)

\bibitem{chen2019deepinspect}
Chen, H., Fu, C., Zhao, J., Koushanfar, F.: Deepinspect: A black-box {T}rojan
  detection and mitigation framework for deep neural networks. In: IJCAI.
  vol.~2, p.~8 (2019)

\bibitem{dumford2020backdooring}
Dumford, J., Scheirer, W.: Backdooring convolutional neural networks via
  targeted weight perturbations. In: International Joint Conference on
  Biometrics. pp.~1--9. IEEE (2020)

\bibitem{esteva2019guide}
Esteva, A., Robicquet, A., Ramsundar, B., Kuleshov, V., DePristo, M., Chou, K.,
  Cui, C., Corrado, G., Thrun, S., Dean, J.: A guide to deep learning in
  healthcare. Nature Medicine  \textbf{25}(1),  24--29 (2019)

\bibitem{gao2020backdoor}
Gao, Y., Doan, B.G., Zhang, Z., Ma, S., Zhang, J., Fu, A., Nepal, S., Kim, H.:
  Backdoor attacks and countermeasures on deep learning: A comprehensive
  review. arXiv preprint arXiv:2007.10760  (2020)

\bibitem{gao2019strip}
Gao, Y., Xu, C., Wang, D., Chen, S., Ranasinghe, D.C., Nepal, S.: Strip: A
  defence against trojan attacks on deep neural networks. In: Proceedings of
  the 35th Annual Computer Security Applications Conference. pp. 113--125
  (2019)

\bibitem{goodfellow2014generative}
Goodfellow, I., Pouget-Abadie, J., Mirza, M., Xu, B., Warde-Farley, D., Ozair,
  S., Courville, A., Bengio, Y.: Generative adversarial nets. Advances in
  {N}eural {I}nformation {P}rocessing {S}ystems  \textbf{27} (2014)

\bibitem{goodfellow2015explain}
Goodfellow, I., Shlens, J., Szegedy, C.: Explaining and harnessing adversarial
  examples. In: International Conference on Learning Representations (2015)

\bibitem{gu2019badnets}
Gu, T., Liu, K., Dolan-Gavitt, B., Garg, S.: Bad{N}ets: {E}valuating
  backdooring attacks on deep neural networks. IEEE Access  \textbf{7},
  47230--47244 (2019)

\bibitem{guo2022aeva}
Guo, J., Li, A., Liu, C.: A{EVA: B}lack-box backdoor detection using
  adversarial extreme value analysis. In: International Conference on Learning
  Representations (2022)

\bibitem{guo2019tabor}
Guo, W., Wang, L., Xing, X., Du, M., Song, D.: T{ABOR}: {A} highly accurate
  approach to inspecting and restoring {T}rojan backdoors in {AI} systems.
  arXiv preprint arXiv:1908.01763  (2019)

\bibitem{hayase2021spectre}
Hayase, J., Kong, W., Somani, R., Oh, S.: {SPECTRE}: defending against backdoor
  attacks using robust statistics. In: 38th International Conference on Machine
  Learning. Proceedings of Machine Learning Research, vol.~139, pp. 4129--4139.
  PMLR (18--24 Jul 2021)

\bibitem{helber2019eurosat}
Helber, P., Bischke, B., Dengel, A., Borth, D.: Introducing {E}uro{SAT}: A
  novel dataset and deep learning benchmark for land use and land cover
  classification. In: IEEE International Geoscience and Remote Sensing
  Symposium. pp. 204--207 (2018)

\bibitem{hinton2015distilling}
Hinton, G., Vinyals, O., Dean, J., et~al.: Distilling the knowledge in a neural
  network. arXiv preprint arXiv:1503.02531  \textbf{2}(7) (2015)

\bibitem{huang2019neuroninspect}
Huang, X., Alzantot, M., Srivastava, M.: Neuron{I}nspect: {D}etecting backdoors
  in neural networks via output explanations. arXiv preprint arXiv:1911.07399
  (2019)

\bibitem{ji2018transferlearningtrojan}
Ji, Y., Zhang, X., Ji, S., Luo, X., Wang, T.: Model-reuse attacks on deep
  learning systems. In: Proceedings {ACM} {C}onference on {C}omputer and
  {C}ommunications {S}ecurity. pp. 349--363 (2018)

\bibitem{kolouri2020universal}
Kolouri, S., Saha, A., Pirsiavash, H., Hoffmann, H.: Universal litmus patterns:
  Revealing backdoor attacks in {CNN}s. In: Proceedings of the IEEE/CVF
  Conference on Computer Vision and Pattern Recognition. pp. 301--310 (2020)

\bibitem{krizhevsky2009CIFARs}
Krizhevsky, A.: Learning multiple layers of features from tiny images. Tech.
  rep., University of Toronto (2009)

\bibitem{krizhevsky2012imagenet}
Krizhevsky, A., Sutskever, I., Hinton, G.E.: Imagenet classification with deep
  convolutional neural networks. In: Advances in {N}eural {I}nformation
  {P}rocessing {S}ystems. vol.~25 (2012)

\bibitem{lecun1998mnist}
LeCun, Y.: The {MNIST} database of handwritten digits. http://yann. lecun.
  com/exdb/mnist/  (1998)

\bibitem{li2021neural}
Li, Y., Lyu, X., Koren, N., Lyu, L., Li, B., Ma, X.: Neural attention
  distillation: Erasing backdoor triggers from deep neural networks.
  International Conference on Learning Representations  (2021)

\bibitem{liu2018fine}
Liu, K., Dolan-Gavitt, B., Garg, S.: Fine-pruning: {D}efending against
  backdooring attacks on deep neural networks. In: International Symposium on
  Research in Attacks, Intrusions, and Defenses. pp. 273--294. Springer (2018)

\bibitem{liu2017neural}
Liu, Y., Xie, Y., Srivastava, A.: Neural {T}rojans. In: International
  Conference on Computer Design. pp. 45--48. IEEE (2017)

\bibitem{panagiota2020trojdrl}
Panagiota, K., Kacper, W., Jha, S., Wenchao, L.: Troj{DRL}: {T}rojan attacks on
  deep reinforcement learning agents. In: ACM/IEEE Design Automation Conference
  (2020)

\bibitem{rajabi2022trojan}
Rajabi, A., Ramasubramanian, B., Poovendran, R.: Trojan horse training for
  breaking defenses against backdoor attacks in deep learning. arXiv preprint
  arXiv:2203.15506  (2022)

\bibitem{rakin2020tbt}
Rakin, A.S., He, Z., Fan, D.: T{BT}: Targeted neural network attack with bit
  {T}rojan. In: IEEE/CVF Conference on Computer Vision and Pattern Recognition.
  pp. 13198--13207 (2020)

\bibitem{sarkar2020backdoor}
Sarkar, E., Alkindi, Y., Maniatakos, M.: Backdoor suppression in neural
  networks using input fuzzing and majority voting. IEEE Design \& Test
  \textbf{37}(2),  103--110 (2020)

\bibitem{shafahi2018poison}
Shafahi, A., Huang, W.R., Najibi, M., Suciu, O., Studer, C., Dumitras, T.,
  Goldstein, T.: Poison frogs! {T}argeted clean-label poisoning attacks on
  neural networks. Advances in {N}eural {I}nformation {P}rocessing {S}ystems
  \textbf{31} (2018)

\bibitem{shen2021backdoor}
Shen, G., Liu, Y., Tao, G., An, S., Xu, Q., Cheng, S., Ma, S., Zhang, X.:
  Backdoor scanning for deep neural networks through k-arm optimization. In:
  International Conference on Machine Learning. pp. 9525--9536. PMLR (2021)

\bibitem{silver2016mastering}
Silver, D., Huang, A., Maddison, C.J., Guez, A., Sifre, L., van~den Driessche,
  G., Schrittwieser, J., Antonoglou, I., Panneershelvam, V., Lanctot, M.,
  Dieleman, S., Grewe, D., Nham, J., Kalchbrenner, N., Sutskever, I.,
  Lillicrap, T., Leach, M., Kavukcuoglu, K., Graepel, T., Hassabis, D.:
  Mastering the game of {G}o with deep neural networks and tree search. Nature
  \textbf{529}(7587),  484--489 (2016)

\bibitem{stallkamp2012man}
Stallkamp, J., Schlipsing, M., Salmen, J., Igel, C.: Man vs. computer:
  Benchmarking machine learning algorithms for traffic sign recognition. Neural
  networks  \textbf{32},  323--332 (2012)

\bibitem{tran2018spectral}
Tran, B., Li, J., Madry, A.: Spectral signatures in backdoor attacks. Advances
  in {N}eural {I}nformation {P}rocessing {S}ystems  \textbf{31} (2018)

\bibitem{ullah2019cyber}
Ullah, F., Naeem, H., Jabbar, S., Khalid, S., Latif, M.A., Al-Turjman, F.,
  Mostarda, L.: Cyber security threats detection in {IoT} using deep learning
  approach. IEEE Access  \textbf{7} (2019)

\bibitem{wang2019neural}
Wang, B., Yao, Y., Shan, S., Li, H., Viswanath, B., Zheng, H., Zhao, B.Y.:
  Neural cleanse: {I}dentifying and mitigating backdoor attacks in neural
  networks. In: IEEE Symposium on Security and Privacy (SP). pp. 707--723
  (2019)

\bibitem{xu2021detecting}
Xu, X., Wang, Q., Li, H., Borisov, N., Gunter, C.A., Li, B.: Detecting {AI}
  {T}rojans using meta neural analysis. In: IEEE Symposium on Security and
  Privacy (SP). pp. 103--120 (2021)

\bibitem{yao2019latent}
Yao, Y., Li, H., Zheng, H., Zhao, B.Y.: Latent backdoor attacks on deep neural
  networks. In: ACM SIGSAC Conference on Computer and Communications Security.
  pp. 2041--2055 (2019)

\bibitem{yoshida2020disabling}
Yoshida, K., Fujino, T.: Disabling backdoor and identifying poison data by
  using knowledge distillation in backdoor attacks on deep neural networks. In:
  ACM Workshop on Artificial Intelligence and Security. pp. 117--127 (2020)

\end{thebibliography}

\end{document}